\documentclass[letterpaper]{article} 
\usepackage{aaai2026}  
\usepackage{times}  
\usepackage{helvet}  
\usepackage{courier}  
\usepackage[hyphens]{url}  
\usepackage{graphicx} 
\usepackage{natbib}  
\usepackage{caption} 
\usepackage{algorithm}
\usepackage{tikz}
\usepackage{booktabs}

\DeclareCaptionStyle{ruled}{labelfont=normalfont,labelsep=colon,strut=off} 
\usepackage[noend]{algpseudocode}
\algrenewcommand\algorithmicindent{1em} 
\usepackage{subcaption}
\usepackage{amsthm} 
\usepackage{placeins}
\usepackage{xspace}
\usepackage{xcolor}
\usepackage{amsmath}
\usepackage{amssymb}
\usepackage{dsfont}

\DeclareMathOperator{\OPT}{OPT}
\DeclareMathOperator{\Reach}{Reach}
\DeclareMathOperator{\first}{first}
\DeclareMathOperator{\last}{last}
\DeclareMathOperator{\cost}{cost}

\newtheorem{example}{Example}
\newtheorem{theorem}{Theorem}[section]
\newtheorem{lemma}[theorem]{Lemma}
\newtheorem{proposition}[theorem]{Proposition}
\newtheorem{corollary}[theorem]{Corollary}

\newtheorem{problem}{Problem}
\newtheorem{assumption}[theorem]{Assumption}

\newtheorem{observation}[theorem]{Observation}

\usepackage{comment}
\newif\ifshowexamples
\showexamplestrue
\ifshowexamples\else\excludecomment{example}\fi

\newif\ifshowfront
\showfronttrue

\newcommand{\NP}{{\small \ensuremath{\mathsf{NP}}\xspace}}

\title{Divide and Collapse:\\
MAPF-Collapse via Exact Decomposition into Independent Sub-Instances}
\author{
    Oren Salzman\textsuperscript{\rm 1}
}
\affiliations{
    \textsuperscript{\rm 1}Technion -- Israel Institute of Technology\\
    osalzman@cs.technion.ac.il
}

\begin{document}

\newcommand\algname[1]{\textsf{#1}\xspace}
\newcommand{\probname}[1]{\textsc{#1}\xspace}
\newcommand{\MAPFC}{\probname{MAPF-Collapse}}
\newcommand{\MAPF}{\probname{MAPF}}
\newcommand{\Judgelight}{\algname{Judgelight}}
\newcommand{\JL}{\algname{JL}}
\newcommand{\CCBS}{\algname{C-CBS}}
\newcommand{\CBS}{\algname{CBS}}
\newcommand{\oneMAPFC}{\algname{1-MAPFC}}
\newcommand{\MAPFCompress}{\algname{Divide-and-Collapse}}
\newcommand{\DnC}{\algname{DnC}}
\newcommand{\DnCJL}{\algname{DnC+JL}}
\newcommand{\DnCCCBS}{\algname{DnC+C-CBS}}
\newcommand{\DnCHyb}{\algname{DnC+C-CBS$\rightarrow$JL}}
\newcommand{\regimetau}{40} 
\newcommand{\Collapse}{\texttt{Collapse}}
\newcommand{\ignore}[1]{}


\newcommand{\myemph}[1]{{\color{teal}\emph{#1}}}

\maketitle

\begin{abstract}
In this work we study the problem of \MAPFC, a post-optimization step for Multi-Agent Path Finding (MAPF) plans where we are given a feasible plan produced by a modern MAPF solver and are tasked with removing avoidable moves while preserving feasibility.
This \NP-hard problem naturally arises when using learning-based state-of-the-art (SOTA) solvers which construct plans that contain redundant moves that can be removed.
Recently, Tang et al.\ presented \Judgelight, which uses Integer Linear Programming (ILP) to solve \MAPFC. Importantly, the ILP is constructed over all agents jointly, so its cost is governed by the full instance rather than by the small coupled residue that actually requires joint reasoning.
Our key insight, motivating this work, is that \MAPFC instances naturally decompose into independent sub-problems, most of which involve a single agent and can be solved without any inter-agent reasoning.
To this end, we first identify which agents need to coordinate their motion and partition the instance into sub-problems accordingly.
For the cases where no coordination is required, we introduce an extremely lightweight solver that is $\approx\!1{,}900\times$ faster than \Judgelight.
For cases where coordination is required, \Judgelight\ can be used but we introduce an alternative \CBS-like solver which is more efficient on easier problems. 
The resulting framework is exact, uses no commercial ILP solver, and matches \Judgelight's quality while running substantially faster on the coordination-light majority of instances; on the coordination-heavy instances we propose a regime-aware hybrid planner that falls back to \Judgelight. Over all benchmarks tested, this planner achieves a median $10.5\times$ per-instance speedup over \Judgelight.
\end{abstract}

\ifshowfront
\section{Introduction}
\label{sec:intro}

\begin{figure*}[t]
\centering
\setlength{\tabcolsep}{4pt}
\begin{subfigure}[b]{0.18\textwidth}
\centering
\begin{tikzpicture}[vert/.style={circle,draw,minimum size=6.5mm,inner sep=0pt,font=\small}]
\node[vert] (b) at (0, 0.5)      {$b$};
\node[vert] (a) at (-1.15, 0.5)  {$a$};
\node[vert] (c) at (1.15, 0.5)   {$c$};
\node[vert] (d) at (-1.15, -0.5) {$d$};
\node[vert] (e) at (0, -0.5)     {$e$};
\node[vert] (f) at (1.15, -0.5)  {$f$};
  \draw (a)--(b);\draw (c)--(b);\draw (f)--(b);\draw (d)--(b);\draw (e)--(b);
\end{tikzpicture}
\caption{Input Graph $G$.}
\label{fig:teaser-graph}
\end{subfigure}
\hfill
\begin{subfigure}[b]{0.24\textwidth}
\centering\small
\begin{tabular}{c|ccccc|c}
\hline
$i$ & $0$ & $1$ & $2$ & $3$ & $4$ & cost\\\hline
$1$ & $a$ & $b$ & $c$ & $b$ & $e$ & $4$\\
$2$ & $b$ & $f$ & $b$ & $f$ & $b$ & $4$\\
$3$ & $d$ & $d$ & $d$ & $d$ & $d$ & $0$\\\hline
\end{tabular}
\caption{Input plan $M$.}
\label{fig:teaser-input}
\end{subfigure}
\hfill
\begin{subfigure}[b]{0.26\textwidth}
\centering
\begin{tikzpicture}[box/.style={draw,rounded corners=2pt,minimum width=6.5mm,minimum height=6.5mm,inner sep=1pt,font=\small}]
  \node[box] (b1) at (0,0)    {$b$};
  \node[box] (c1) at (0.75,0) {$c$};
  \node[box] (b2) at (1.5,0)  {$b$};
  \draw[->] (b1)--(c1);\draw[->] (c1)--(b2);
  \node[anchor=west,font=\footnotesize] at (1.95,0) {$2$ moves};
  \draw[->] (0.75,-0.4)--(0.75,-0.62);
  \node[box,fill=teal!12,draw=teal] (b3) at (0,-1)    {$b$};
  \node[box,fill=teal!12,draw=teal] (b4) at (0.75,-1) {$b$};
  \node[box,fill=teal!12,draw=teal] (b5) at (1.5,-1)  {$b$};
  \node[anchor=west,font=\footnotesize] at (1.95,-1) {$0$ moves};
\end{tikzpicture}
\caption{\Collapse\ operator.}
\label{fig:teaser-op}
\end{subfigure}
\hfill
\begin{subfigure}[b]{0.24\textwidth}
\centering\small
\begin{tabular}{c|ccccc|c}
\hline
$i$ & $0$ & $1$ & $2$ & $3$ & $4$ & cost\\\hline
$1$ & $a$ & \textcolor{teal}{$b$} & \textcolor{teal}{$b$} & \textcolor{teal}{$b$} & $e$ & $2$\\
$2$ & $b$ & \textcolor{teal}{$f$} & \textcolor{teal}{$f$} & \textcolor{teal}{$f$} & $b$ & $2$\\
$3$ & $d$ & $d$ & $d$ & $d$ & $d$ & $0$\\\hline
\end{tabular}
\caption{Optimum plan $\Pi^\star$.}
\label{fig:teaser-output}
\end{subfigure}
\caption{
    \MAPFC\ problem.
    (a,b) Input graph $G$ and plan $M$ of cost $8$.
    The problem calls for applying the \Collapse\ operator (c) which replaces a closed walk that returns to a vertex, here $b\to c\to b$, by waiting at $b$. 
    (d)~Applying collapses yields the optimum $\Pi^\star$ of cost $4$; \textcolor{teal}{teal} cells are moves turned into waits. 
    Noteworthy is that agents $1$ and $2$ contend for $b$ and form one coupled component solved jointly, whereas agent $3$ never interacts which is the structure our framework exploits.}
\label{fig:teaser}
\end{figure*}
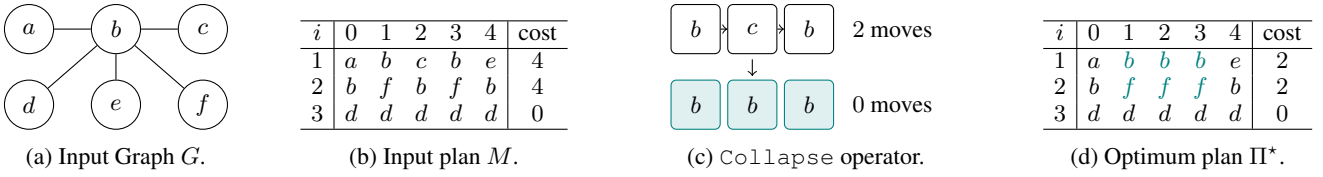

In Multi-Agent Path Finding (\MAPF)~\citep{Stern2019Definitions,Salzman2020Challenges}, we are tasked with coordinating the motion of a fleet of agents. It has applications in automated warehouses, robotic fulfillment, and large-scale logistics fleets~\citep{Wurman2008Kiva,Sturtevant2012MovingAI,Skrynnik2024POGEMA}. Modern \MAPF\ solvers, whether search-based~\citep{Okumura2023LaCAM} or learned~\citep{Sartoretti2019PRIMAL,Wang2023SCRIMP,Skrynnik2024Follower,Andreychuk2025MAPFGPT}, are tuned to produce \emph{feasible} schedules quickly. The trajectories they emit routinely contain moves that could be replaced by waiting in place without violating feasibility, inflating energy and battery consumption with no throughput benefit. \citet{Tang2026Judgelight} recently formalized the corresponding post-optimization task as \MAPFC (see Fig.~\ref{fig:teaser} and Sec.~\ref{sec:formulation}): given a feasible plan $M$, find a feasible modification of $M$ minimizing total moves by \myemph{collapsing} an agent's closed subwalk into a sequence of waits at its endpoint.

Unfortunately, \MAPFC\ is \NP-hard~\citep[Thm.~1]{Tang2026Judgelight}. To solve \MAPFC,
Tang et al. presented \Judgelight (Sec.~\ref{sec:judgelight}), which formulates the problem as an Integer Linear Program (ILP) over all agents jointly.
Our key insight, motivating this work, is that most agents in typical \MAPFC\ instances never occupy the same position at the same time and can therefore be post-optimized in complete isolation. Decomposing \MAPF\ into independent groups is a classical idea~\citep{Standley2010ID,Wagner2015Mstar,Sharon2012MACBS}; what \MAPFC\ adds is enough structure to make this split \myemph{exact} and compute it up front; to the best of our knowledge, this is the first exact, up-front decomposition for a \MAPF\ post-optimization problem.

Concretely, we propose \MAPFCompress (Sec.~\ref{sec:method}), a two-stage framework which starts by partitioning the agents into groups according to whether they could possibly occupy the same position at the same time under any sequence of collapses. 
Each isolated agent's problem then reduces to the simple question of which of its moves can safely be turned into waits, which we answer optimally and in linear time (Sec.~\ref{sec:single-agent}). Only the small residue of agents that actually need to coordinate requires a joint solver. 
When agents need to be coordinated, any \MAPFC\ planner such as \Judgelight can be used. However, a structural property of \MAPFC\ makes the \CBS\ framework a natural fit: the optimization stage is only required to reason about vertex collisions (and not edge collisions as in the general \MAPF\ problem). Consequently, we suggest (Sec.~\ref{sec:improvements}) Collapse-\CBS\ (\CCBS),\footnote{Not to be confused with \algname{CCBS}, Continuous-time \CBS~\citep{Andreychuk2022CCBS}.} a \CBS-like solver that lets us harness the state-of-the-art algorithmic tools developed for \CBS.
We then continue to evaluate our algorithmic framework (Sec.~\ref{sec:experiments}) and conclude with a discussion regarding future work (Sec.~\ref{sec:future-work}). 

Our evaluation tests the framework with both \CCBS\ and \Judgelight\ as the joint solver, and uncovers a phase transition: \CCBS\ is faster by an order of magnitude on small components, while using \Judgelight\ within the framework becomes beneficial only on large-problem instances. This phase transition prescribes a \myemph{regime-aware} planner: each non-trivial component is dispatched according to the regime size, \CCBS\ (with a \Judgelight\ fallback) on the small-to-medium majority, and \Judgelight\  on the largest cores. The resulting planner is exact on every component \CCBS\ solves, never returns a costlier plan than \Judgelight, and often a cheaper one, at a $10.5\times$ median per-instance speedup. On the coordination-light majority it is near-instant; and only on a fraction of the densest instances does it fall back to \Judgelight.

\ignore{
\paragraph{Contributions.} Our contributions include (i)~the \MAPFCompress\ framework, comprising an exact-decomposition theorem (Thm.~\ref{thm:exact-decomposition}) and a near-linear-time decomposition pipeline (Alg.~\ref{alg:naive}) that factorize \MAPFC\ across the connected components of an interaction graph $H$, a linear-time per-agent primitive (Lem.~\ref{lem:smart-sweep}, Cor.~\ref{cor:single-agent-optimality2}) for singleton dispatch, and an exact \CBS-style joint solver \CCBS\ (Alg.~\ref{alg:ccbs}) for the non-trivial residue; and (ii)~an empirical evaluation on POGEMA-derived data showing that the hybrid configuration \MAPFCompress\,+\,\CCBS\,$\to$\,\Judgelight\ solves all $3{,}296$ instances ($3{,}296/3{,}296$) with $599/3{,}296$ ($18\%$) strict cost wins over \Judgelight, zero cost regressions, and a $10.5\times$ median per-instance speedup.
}

\section{Related Work}
\label{sec:related}

Our work sits at the intersection of three lines of research: learning-based and lifelong \MAPF, post-optimization of \MAPF\ schedules, and heuristic search for \MAPF, which includes both the conflict-based solvers we adapt and the agent-decomposition methods closest to our approach.

\subsection{Learning-Based and Lifelong \MAPF}
Learning-based \MAPF\ solvers, such as \algname{PRIMAL}~\citep{Sartoretti2019PRIMAL}, \algname{SCRIMP}~\citep{Wang2023SCRIMP}, \algname{Follower}~\citep{Skrynnik2024Follower}, \algname{MAPF-GPT}~\citep{Andreychuk2025MAPFGPT}
and
\algname{RAILGUN}~\citep{Tang2025RAILGUN}, train decentralized policies that scale to fleets and horizons beyond the reach of optimal search, producing feasible schedules in milliseconds. This speed comes at a price: with no optimality mechanism, the emitted plans contain redundant moves such as back-and-forth oscillations, reflected in the cost gaps these solvers exhibit relative to search-based baselines on the \algname{POGEMA} benchmark~\citep{Skrynnik2024POGEMA}. The same holds for fast rule-based solvers such as \algname{PIBT}~\citep{Okumura2022PIBT}, whose plans often contain similar oscillations. Moreover, these policies increasingly target lifelong settings~\citep{Skrynnik2024Follower}, where plans are continually re-issued and per-plan waste compounds across the replanning loop.

\subsection{Post-Optimization of \MAPF\ Schedules}
Post-processing of \MAPF\ plans has been studied along several axes. Large-neighborhood search~\citep{Li2021MAPFLNS} and iterative refinement~\citep{Okumura2021Refinement} destroy and repair selected agent trajectories to improve cost while preserving feasibility, and delay-introduction~\citep{Kottinger2024Delays} adds controlled waits to account for  execution-time failures rather than to reduce cost.

A related execution-level line of work re-times a fixed spatial plan, for kinematic feasibility~\citep{Hoenig2016POST}, robust execution under delays~\citep{Hoenig2019ADG,Ma2017MCP}, or built-in slack~\citep{Atzmon2018Robust}. Similarly, \algname{APEX-MR}~\citep{Huang2025APEXMR} post-processes a sequential multi-robot task plan into an asynchronous execution plan that absorbs delays and contingencies. All these methods operate on the temporal (execution) layer while  \MAPFC\ operates on the path (movement) layer: it preserves the timeline and shortens the spatial trace, turning moves into waits to reduce cost.
Tailored specifically to \MAPFC\ is the recent ILP-based \Judgelight\ of \citet{Tang2026Judgelight}, which we recap in Sec.~\ref{sec:judgelight}.

\subsection{Heuristic Search for \MAPF}
Heuristic search has been the algorithmic infrastructure used to develop state-of-the-art algorithms for \MAPF, with Conflict-Based Search (\CBS)~\citep{Sharon2015CBS} being one of the prominent examples. 
\CBS\ computes optimal solutions by exploring a high-level tree whose nodes carry per-agent constraint sets and dispatching the per-agent search to a low-level solver, resolving each vertex conflict in the resulting joint plan by branching. Its running time has been improved by (i)~relaxing optimality guarantees to bounded suboptimality~\citep{Barer2014ECBS,Li2021EECBS} or by (ii)~introducing optimizations to the basic version of \CBS\ such as cardinal-conflict prioritization and conflict bypass~\citep{Boyarski2015ICBS}, disjoint splitting~\citep{Li2019DisjointSplitting}, and symmetry breaking~\citep{Li2019Symmetry}.

A common tool used in the development of heuristic search algorithms for \MAPF\ is reasoning about which agents actually interact. 
Independence detection~\citep{Standley2010ID} partitions the agents into groups, plans each group in isolation, and merges two groups only once a conflict between them is detected, re-planning the merged group; 
\algname{M*}~\citep{Wagner2015Mstar} couples agents lazily by inflating the search dimension only around realized collisions; 
and meta-agent \CBS~\citep{Sharon2012MACBS} merges agents into a meta-agent inside the \CBS\ tree once their conflict count crosses a threshold. 
The common thread is that the coupling structure is discovered \emph{lazily}, during search, from conflicts that actually arise.
The closest line of research to ours is this family of decomposition methods, from which we differ in that (i)~we reason about agent interaction up front, before running the planner, rather than lazily during search; and (ii)~our interaction reasoning is tailored to \MAPFC, whose reach sets make the split exact.

\fi
\section{Preliminaries and Problem Formulation}
\label{sec:formulation}

Let $G=(V,E)$ be an undirected graph. A \myemph{single-agent plan} is a sequence $\pi=\langle v_0,v_1,\ldots,v_k\rangle$ of vertices in $V$ such that, for every $0\le t<k$, either $v_t=v_{t{+}1}$ (the agent \myemph{waits} at $v_t$) or $(v_t,v_{t{+}1})\in E$ (the agent \myemph{moves} along an edge). We write $|\pi|:=k$ for the \myemph{length} of $\pi$ and $\pi(t):=v_t$ for the agent's position at time $t\in\{0,\ldots,k\}$. The \myemph{cost} of $\pi$ is its move count,
$$
\cost(\pi) \;:=\; \bigl|\bigl\{\,t\in\{0,\ldots,k-1\}\,\big\vert\,\pi(t)\neq\pi(t{+}1)\,\bigr\}\bigr|.
$$

Two single-agent plans $\pi^i,\pi^j$ are in \myemph{vertex collision} at time $t$ if $\pi^i(t)=\pi^j(t)$, and in \myemph{edge collision} at time $t$ if $\pi^i(t)\ne\pi^i(t{+}1)$, $\pi^i(t)=\pi^j(t{+}1)$, and $\pi^i(t{+}1)=\pi^j(t)$. The pair is \myemph{collision-free} if no such $t$ exists.

For agents $I=\{1,\ldots,N\}$, a \myemph{joint plan} is a tuple $\Pi=(\pi^1,\ldots,\pi^N)$ of single-agent plans. Its \myemph{end-time} is $T(\Pi):=\max_{i\in I}|\pi^i|$; each agent's position is extended by waits at its endpoint, namely with $v_{i,\text{end}}:=\pi^i(|\pi^i|)$ we set $\pi^i(t):=v_{i,\text{end}}$ for $|\pi^i|<t\le T(\Pi)$. The joint plan is \myemph{collision-free} if every pair $(\pi^i,\pi^j)$, $i\neq j$, is collision-free over $\{0,\ldots,T(\Pi)\}$. Its \myemph{cost} is $\cost(\Pi):=\sum_{i\in I}\cost(\pi^i)$.

Given a single-agent plan $\pi=\langle v_0,\ldots,v_k\rangle$ and indices $0\le a<b\le k$ with $v_a=v_b$, the \myemph{collapse} operation $\Collapse(\pi,a,b)$ returns the plan $\pi'$ of length~$k$ defined by $\pi'(t):=v_a$ for every $t\in[a,b]$ and $\pi'(t):=v_t$ otherwise. We call the common endpoint vertex $v_a=v_b$ the \myemph{anchor} of $\Collapse(\pi,a,b)$.
As the operation replaces a sequence of moves by waits we have that $\cost(\Collapse(\pi,a,b))\le\cost(\pi)$.

\begin{example}\label{ex:single-collapse}
Consider the graph $G$ of Fig.~\ref{fig:teaser}, and consider the single-agent plan $\pi=\langle a,b,c,b,e\rangle$ on $G$. This plan has length~$4$ and $\cost(\pi)=4$: every consecutive pair differs. The only repeated vertex is $b$, occurring at $t\in\{1,3\}$, so the only admissible $\Collapse$ application is $\Collapse(\pi,1,3)$, which yields $\pi^\star=\langle a,b,b,b,e\rangle$ with $\cost(\pi^\star)=2$; this is the minimum-cost plan obtainable from $\pi$ by a sequence of $\Collapse$ applications.
\end{example}

The main problem of this paper extends this single-agent setting to multiple agents under joint-feasibility constraints.

\begin{problem}[\MAPFC]\label{prob:mapfc}
Given a graph $G$ and a collision-free joint plan $M=(M^i)_{i\in I}$ on $G$ for agents~$I$,\footnote{We reserve $M$ for the input joint plan and $\Pi$ for any joint plan under construction or returned by the algorithm; both are joint plans in the sense defined above. The distinction lets every occurrence of $M^i$ refer unambiguously to the input row of agent~$i$ and every $\pi^i$ to a candidate single-agent plan derived from it.} find a collision-free joint plan $\Pi^\star=(\pi^{\star i})_{i\in I}$ in which each~$\pi^{\star i}$ is obtained from $M^i$ by a finite sequence of $\Collapse$ applications, that minimizes $\cost(\Pi^\star)$.
\end{problem}

For a subset of agents $S\subseteq I$, we write $\OPT(S)$ for the cost of an optimal solution to \MAPFC\ restricted to $S$ with the corresponding rows of $M$. We further write~$k_i:=|M^i|$ for the length of agent $i$'s input plan.

\begin{assumption}[common makespan]\label{asm:common-makespan}
The input rows share a common makespan: $|M^i|=T(M)$ for all $i\in I$, equivalently $k_i=T(M)$.
\end{assumption}
This is the standard form of a \MAPF\ solution, in which each agent waits at its goal until the last agent arrives.

\begin{example}\label{ex:setup}
Consider the graph $G$ and the input plan $M$ of Fig.~\ref{fig:teaser}. The input is collision-free (no shared cell, no swap) and $\cost(M)=4+4+0=8$. Note that the per-agent optimum on $M^1$ in isolation (Ex.~\ref{ex:single-collapse}) keeps agent~$1$ waiting at~$b$ during $[1,3]$, while $M^2$ also visits $b$ at $t=2$; whether the single-agent optima are jointly feasible is precisely the coupling question that distinguishes \MAPFC\ from its single-agent restriction.
\end{example}

\section{Algorithmic Background: \Judgelight}
\label{sec:judgelight}

\citet{Tang2026Judgelight} solve \MAPFC\ using \Judgelight, which casts the problem as an ILP after two preprocessing reductions.

\Judgelight enumerates every candidate collapse action by introducing a binary variable $y_c\in\{0,1\}$ for every quadruple $c=(i,a,b,v)$ with $M^i(a)=M^i(b)=v$, where $y_c=1$ indicates that the collapse $(a,b)$ is applied to $M^i$. Each variable carries a non-negative weight $w_c$ equal to the number of moves of $M^i$ inside $[a,b]$, the saving the collapse realizes in isolation. The ILP maximizes $\sum_c w_c\,y_c$ subject to three constraint families: \myemph{per-agent exclusion} (two selected actions of one agent must have time-disjoint intervals), \myemph{vertex-collision exclusion} (no two selected actions place two agents at the same cell at the same time), and \myemph{dependency constraints} (if a selected action keeps one agent at a cell that another agent originally occupies, that other agent must collapse around the cell). The dependency constraints are the inter-agent coupling that makes the problem hard. In contrast, our \CCBS\ discovers this coupling lazily within each component, one resolved conflict at a time (Sec.~\ref{sec:improvements}).

Two reductions run before the ILP is built. \myemph{Safe-oscillation removal} greedily eliminates every oscillation an agent can undo on its own, without affecting any other agent; because the choice is greedy, it can commit an agent to waits that preclude a larger joint saving. A \myemph{maximal-loop reduction} losslessly prunes redundant collapse candidates, shrinking the variable set. Even after this reduction, an alternating segment of length $n$ still yields $O(n^2)$ variables and $O(n^4)$ dependency constraints, so building the ILP is worst-case quartic in an agent's plan length.

\section{Decomposing \MAPFC by Interaction Components}
\label{sec:method}

In this section we present our approach to solving \MAPFC\ (Prob.~\ref{prob:mapfc}). We begin (Lem.~\ref{lem:no-edge-conflicts}) with a structural property of joint \Collapse\ applied to a collision-free input: no edge collision can ever arise, so the only feasibility constraint we need to enforce is vertex-collision avoidance. We then introduce the \myemph{reach set} of an agent, namely the (vertex, time) cells it can occupy under any plan derived from $M^i$, and use pairwise reach-set intersections to define an \myemph{interaction graph} on the agents whose edges record which pairs can potentially conflict. These two ingredients let us decompose any \MAPFC\ instance into independent sub-instances, one per connected component of the interaction graph (Sec.~\ref{subsec:pipeline}): singletons are dispatched to a per-agent primitive \oneMAPFC (Sec.~\ref{sec:single-agent}), that returns the per-agent optimum $\OPT(\{i\})$ in time linear in~$|M^i|$; non-trivial components are handed to a joint solver of choice~(Sec.~\ref{sec:improvements}). 

\begin{lemma}[no edge conflicts]\label{lem:no-edge-conflicts}
Let $M$ be a collision-free joint plan and $\Pi$ a joint plan in which each $\pi^i$ is obtained from $M^i$ by a finite sequence of $\Collapse$ applications. Then $\Pi$ contains no edge collision.
\end{lemma}

\begin{proof}
We call a triple $(t,\pi(t),\pi(t{+}1))$ with $\pi(t)\ne\pi(t{+}1)$ a \myemph{move-transition} of $\pi$. A single $\Collapse$ turns the moves at timesteps inside its interval into waits and leaves all other timesteps unchanged, so it can only shrink a plan's move-transition set; by induction, the move-transitions of each $\pi^i$ are a subset of those of $M^i$. An edge collision between agents $i$ and~$j$ at time $t$ consists of a move-transition $(t,u,v)$ of agent~$i$ and the opposite move-transition $(t,v,u)$ of agent~$j$; these would be move-transitions of $M^i$ and~$M^j$ as well, so $M$ would contain the same edge collision, contradicting that $M$ is collision-free.
\end{proof}

Throughout the remainder of the paper, ``conflict'' therefore always means a vertex collision.

\subsection{Reach Sets and the Interaction Graph}
\label{subsec:reach-and-H}
Recall that~$T(M):=\max_i|M^i|$ denotes the \myemph{end-time} of the input joint plan (Sec.~\ref{sec:formulation}). A \myemph{vertex-time cell} is a pair $(v,t)\in V\times\{0,\ldots,T(M)\}$. For an agent $i$ and a vertex $v\in V$ visited by $M^i$, let
$$
\begin{aligned}
\first_i(v) &\;:=\;\min\{t:M^i(t)=v\}, \\
\last_i(v)  &\;:=\;\max\{t:M^i(t)=v\}.
\end{aligned}
$$

\begin{figure}[t]
\centering
\begin{tikzpicture}[
  cell/.style={draw, minimum size=8mm, inner sep=0pt, font=\small},
  hl/.style={draw, fill=gray!25, minimum size=8mm, inner sep=0pt, font=\small},
  x=8mm, y=8mm
]
\foreach \t in {0,1,2,3,4} {
  \node[font=\small] at (\t, 0.65) {$t{=}\t$};
}
\foreach \r/\lab in {0/a, -1/b, -2/c, -3/d, -4/e, -5/f} {
  \node[font=\small] at (-0.75, \r) {$\lab$};
}
\node[cell] at (0, 0)  {$1$};
\node[cell] at (1, 0)  {$\,$};
\node[cell] at (2, 0)  {$\,$};
\node[cell] at (3, 0)  {$\,$};
\node[cell] at (4, 0)  {$\,$};
\node[cell] at (0, -1) {$2$};
\node[hl]   at (1, -1) {$1{,}2$};
\node[hl]   at (2, -1) {$1{,}2$};
\node[hl]   at (3, -1) {$1{,}2$};
\node[cell] at (4, -1) {$2$};
\node[cell] at (0, -2) {$\,$};
\node[cell] at (1, -2) {$\,$};
\node[cell] at (2, -2) {$1$};
\node[cell] at (3, -2) {$\,$};
\node[cell] at (4, -2) {$\,$};
\foreach \t in {0,1,2,3,4} {
  \node[cell] at (\t, -3) {$3$};
}
\node[cell] at (0, -4) {$\,$};
\node[cell] at (1, -4) {$\,$};
\node[cell] at (2, -4) {$\,$};
\node[cell] at (3, -4) {$\,$};
\node[cell] at (4, -4) {$1$};
\node[cell] at (0, -5) {$\,$};
\node[cell] at (1, -5) {$2$};
\node[cell] at (2, -5) {$2$};
\node[cell] at (3, -5) {$2$};
\node[cell] at (4, -5) {$\,$};
\end{tikzpicture}
\caption{Reach-set for the running example. Each cell~$(v,t)$ is labeled with the agents whose reach set $\Reach^i$ includes it; an empty cell is in no reach set. The shaded cells
form the intersection $\Reach^1\cap\Reach^2$ that witnesses the interaction edge $\{1,2\}$ of $H$ (Fig.~\ref{fig:interaction-graph}). Agent~$3$'s column (row~$d$) is disjoint from the others, so $\{3\}$ is a singleton component.}
\label{fig:reach-grid}
\end{figure}

An \myemph{anchor} of $M^i$ is a vertex $v$ with $\first_i(v)<\last_i(v)$, namely a vertex visited at least twice. We define the \myemph{original cells}~$\Reach^i_{\text{orig}}$ traced by~$M^i$ and the \myemph{anchor cells}~$\Reach^i_{\text{anchor}}$ visited when $i$ collapses as follows:
$$
\begin{aligned}
\Reach^i_{\text{orig}} &\;=\;\bigl\{(M^i(t),t)\,\big\vert\,0\le t\le T(M)\bigr\}, \\
\Reach^i_{\text{anchor}} &\;=\;\bigcup_{v\,:\,\first_i(v)<\last_i(v)}\{v\}\times\bigl[\first_i(v),\last_i(v)\bigr].
\end{aligned}
$$
The \myemph{reach set} of agent $i$ is their union:
$$
\Reach^i\;=\;\Reach^i_{\text{orig}}\,\cup\,\Reach^i_{\text{anchor}}.
$$
Roughly speaking, $\Reach^i$ is the set of all cells agent $i$ can occupy under some single-agent plan derived from $M^i$, which we now make precise.

\begin{observation}[reach containment]\label{obs:reach-containment}
Every $\Collapse$-derived plan $\pi^i$ of $M^i$ satisfies $(\pi^i(t),t)\in\Reach^i$ for every $0\le t\le T(M)$. Conversely, every cell of $\Reach^i$ is occupied by some $\Collapse$-derived plan of~$M^i$.
\end{observation}

\begin{proof}
We prove the first claim by induction on the number of $\Collapse$ applications, with the invariant that at every time~$t$ either $\pi^i(t)=M^i(t)$ or $\pi^i(t)=w$ for an anchor~$w$ of~$M^i$ with $\first_i(w)\le t\le\last_i(w)$; the invariant places $(\pi^i(t),t)$ in $\Reach^i_{\text{orig}}$ or in $\Reach^i_{\text{anchor}}$, respectively. The base case is $M^i$ itself. For the inductive step, let $\pi'=\Collapse(\pi,a,b)$ with anchor $v=\pi(a)=\pi(b)$, where $\pi$ satisfies the invariant. Applying the invariant at time~$a$ gives $\first_i(v)\le a$, both when $v=M^i(a)$, as then $M^i$ visits~$v$ at time~$a$, and when $v$ is an anchor of~$M^i$; symmetrically, at time~$b$ it gives $b\le\last_i(v)$. Hence $\first_i(v)\le a<b\le\last_i(v)$, so~$v$ is an anchor of~$M^i$ and $[a,b]\subseteq[\first_i(v),\last_i(v)]$. Every cell $\pi'$ newly occupies therefore lies in $\Reach^i_{\text{anchor}}$, and outside $[a,b]$ the plan is unchanged, so $\pi'$ satisfies the invariant. For the converse, $M^i$ occupies every cell of $\Reach^i_{\text{orig}}$, and for an anchor~$v$ the single application $\Collapse(M^i,\first_i(v),\last_i(v))$ occupies $(v,t)$ for every $t\in[\first_i(v),\last_i(v)]$.
\end{proof}

Noteworthy is that the anchor of a $\Collapse$ application on an already-collapsed plan is itself an anchor of the input row~$M^i$, which is why the reach set does not grow as applications compose.
See Fig.~\ref{fig:reach-grid} for a visualization of the reach sets of the running example.

\begin{example}\label{ex:reach}
Take agent~$1$ of the running example, with input plan $M^1=\langle a,b,c,b,e\rangle$ (Ex.~\ref{ex:single-collapse}). Its original cells are the five it traces, $\Reach^1_{\text{orig}}=\{(a,0),(b,1),(c,2),(b,3),(e,4)\}$. Its only anchor is $b$ resulting in $\Reach^1_{\text{anchor}}=\{b\}\times[1,3]=\{(b,1),(b,2),(b,3)\}$. The reach set $\Reach^1$ is their union, the cells marked~$1$ in Fig.~\ref{fig:reach-grid}.
\end{example}

The \myemph{interaction graph}~$H=(I,\mathcal E)$ is the undirected graph on agents defined by
$$
\{i,j\}\in\mathcal E\iff\Reach^i\cap\Reach^j\ne\emptyset.
$$
For any distinct $i,j$, $\Reach^i_{\text{orig}}\cap\Reach^j_{\text{orig}}=\emptyset$, since these are exactly the cells of the collision-free input joint plan $M$; hence every edge of $H$ is witnessed by at least one anchor cell. A component $I_\ell$ of $H$ with $|I_\ell|=1$ is a \myemph{singleton}; one with $|I_\ell|\ge 2$ is \myemph{non-trivial}.

\begin{figure}[t]
\centering
\begin{tikzpicture}[
  vert/.style={circle, draw, minimum size=8mm, inner sep=0pt}
]
\node[vert] (a1) at (0, 0)   {$1$};
\node[vert] (a2) at (1.5, 0) {$2$};
\node[vert] (a3) at (3, 0)   {$3$};
\draw (a1) -- (a2);
\end{tikzpicture}
\caption{The interaction graph $H$ for the running example. Component $\{1,2\}$ is non-trivial (witnessed by $(b,1),(b,2),(b,3)\in\Reach^1\cap\Reach^2$, the shaded cells of Fig.~\ref{fig:reach-grid}); agent~$3$ is a singleton.}
\label{fig:interaction-graph}
\end{figure}
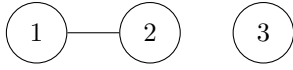

\subsection{Exact Decomposition and the Pipeline}
\label{subsec:pipeline}

Let $I_1,\ldots,I_m$ be the connected components of $H$. The next theorem shows that the reach sets and the interaction graph suffice to factorize the problem across these components.

\begin{theorem}[exact decomposition]\label{thm:exact-decomposition}
Combining optimal plans for the individual components yields a globally optimal plan. In particular, $\OPT(I)=\sum_{\ell=1}^{m}\OPT(I_\ell)$.
\end{theorem}

\begin{proof}
By Lem.~\ref{lem:no-edge-conflicts}, feasibility reduces to absence of vertex collisions. If agents $i$ and $j$ collide at cell $(v,t)$ under some pair of plans, then $(v,t)\in\Reach^i\cap\Reach^j$ by Obs.~\ref{obs:reach-containment}. Thus $\{i,j\}\in\mathcal E$; equivalently, any colliding pair is $H$-adjacent and lies in a single component. Hence cross-component pairs, and non-adjacent same-component pairs, are collision-free under every choice of plans, and a joint plan is feasible iff its restriction to each component is feasible. The cost decomposes additively over agents and, by the factorization above, the components can be optimized independently; the global minimum therefore equals the sum of per-component minima, proving the decomposition.
\end{proof}

Thm.~\ref{thm:exact-decomposition} suggests a natural algorithmic realization, outlined in Alg.~\ref{alg:naive}.
We start by computing the reach set of every agent (Lines~1--4). 
Specifically, for each agent~$i$, a single left-to-right pass over $M^i$ emits the trajectory cells of $\Reach^i_{\text{orig}}$ and records, in a vertex hash $\mathcal H_i$, the first and last visit times $\first_i(v),\last_i(v)$ of each visited vertex~$v$ (Line~3). Every anchor of $M^i$ then contributes its band of anchor cells $\{v\}\times[\first_i(v),\last_i(v)]$. The resulting reach sets are inserted, cell by cell, into a shared cell-to-agents map $\mathcal M$ (Line~4).
We then recover the interaction graph and its components (Lines~5--6). As stated, Line~5 enumerates the edge set $\mathcal E$, which a single cell shared by $c$ agents would inflate by $\binom{c}{2}$ pairs; this is avoided by maintaining a union--find over the agents: whenever an agent lands on a cell of $\mathcal M$ already occupied by another, we union the two. The connected components $I_1,\ldots,I_m$ are read off as the union--find classes (Line~6); the bound of Prop.~\ref{prop:pipeline-complexity} refers to this realization.
Finally, we solve the components independently and stitch the results (Lines~7--14). A singleton $\{i\}$ goes to the per-agent primitive \oneMAPFC\ (Line~10), while a non-trivial component is handed to a joint solver \textsc{JointSolve} as the \myemph{self-contained sub-instance} $(G,M_{I_\ell})$, the restriction of $M$ to the agents in~$I_\ell$ (Line~12). Appending each component's plans (Line~13) gives the joint plan $\Pi$ returned in Line~14; by Thm.~\ref{thm:exact-decomposition}, $\Pi$ is globally optimal whenever every component is solved optimally.

\begin{algorithm}[t]
\caption{\MAPFCompress: decomposition pipeline.}\label{alg:naive}
\begin{algorithmic}[1]
\Require \MAPFC\ instance $(G,M)$
        on agents~$I$.
\Ensure A feasible joint plan $\Pi$.
\State $\mathcal M\gets\emptyset$\Comment{cell $\to$ agents map}
\ForAll{$i\in I$}
  \State Compute $\first_i(v),\last_i(v)$ for each anchor $v$ of $M^i$
  \State $\Reach^i\gets\Reach^i_{\text{orig}}\cup\Reach^i_{\text{anchor}}$; insert its cells into $\mathcal M$
\EndFor
\State $\mathcal E\gets\{\{i,j\}\,\vert\,\exists(v,t)\text{ with }\{i,j\}\subseteq\mathcal M[(v,t)]\}$
\State $\{I_1,\ldots,I_m\}\gets$ connected components of $H{=}(I,\mathcal E)$
\State $\Pi\gets()$
\ForAll{component $I_\ell$}
  \If{$|I_\ell|=1$, say $I_\ell{=}\{i\}$}
    \State $\pi^i\gets\oneMAPFC\!(G,M^i)$\Comment{Sec.~\ref{sec:single-agent}, Cor.~\ref{cor:single-agent-optimality2}}
  \Else
    \State $\Pi_{I_\ell}\gets$\textsc{JointSolve}$(G,M_{I_\ell})$
  \EndIf
  \State Append the per-agent plans to $\Pi$
\EndFor
\State \Return $\Pi$
\end{algorithmic}
\end{algorithm}

\begin{proposition}[decomposition-pipeline complexity]\label{prop:pipeline-complexity}
The singleton-exact part of Alg.~\ref{alg:naive}, namely all of its steps except the \textsc{JointSolve} call on Line~12, runs in time\footnote{$\tilde{O}(\cdot)$ hides the inverse-Ackermann factor $\alpha(N)$ contributed by the union--find (see proof), the sole super-linear factor in the bound.}
$$
\tilde{O}\!\Bigl(\sum_i k_i\;+\;\sum_i|\Reach^i|\Bigr).
$$
\end{proposition}

\begin{proof}
We bound the work agent by agent. Fix an agent~$i$ with $k_i=|M^i|$. A single left-to-right pass over $M^i$ computes $\first_i(v),\last_i(v)$ for every visited vertex and emits the original cells $\Reach^i_{\text{orig}}$, in $O(k_i)$ time (Lines~3--4). Emitting the anchor cells $\Reach^i_{\text{anchor}}$ and inserting all $|\Reach^i|$ cells of $\Reach^i$ into the shared map $\mathcal M$ takes $O(|\Reach^i|)$ time; each insertion performs one hash lookup and triggers at most one $\textsc{union}$, contributing $O(|\Reach^i|\,\alpha(N))$ union--find work~\citep{Tarjan1975SetUnion}, where~$\alpha$ is the extremely slowly growing inverse-Ackermann function\footnote{We have that $\alpha(N)\le 4$ for every $N$ of practical interest.}  (Lines~4--6). Each singleton is then solved by one \oneMAPFC\ call in $O(k_i)$ time (Line~10; Cor.~\ref{cor:single-agent-optimality2}). Summing over the $N$ agents,
$$
\begin{aligned}
&\sum_i\Bigl(O(k_i)+O\bigl(|\Reach^i| \cdot  \alpha(N)\bigr)\Bigr)\\
&\quad=\;O\!\Bigl(\sum_i k_i+\alpha(N) \cdot \sum_i|\Reach^i|\Bigr),
\end{aligned}
$$
which is the claimed $\tilde{O}\bigl(\sum_i k_i+\sum_i|\Reach^i|\bigr)$ once the inverse-Ackermann factor is absorbed into $\tilde{O}$.
\end{proof}

The bound is output-sensitive: $|\Reach^i|$ can reach $\Theta(k_i^2)$ in the worst case, though in all our experiments (Sec.~\ref{sec:experiments}) $\sum_i|\Reach^i|$ never exceeds $4\times\sum_i k_i$ on any instance.

The pipeline yields two bounds at no extra cost. The input joint plan $M$ is a global \myemph{upper bound} $\cost(M)\ge\OPT(I)$. The sum of unconstrained single-agent optima is a global \myemph{lower bound}
$
\OPT(I)\;\ge\;\sum_i\OPT(\{i\})
$,
since coordination can only raise an agent's cost. As we will see, these bounds seed the root of the joint solver of Sec.~\ref{sec:improvements}.

\begin{example}[continuing Ex.~\ref{ex:reach}]\label{ex:naive}
On the running example, Alg.~\ref{alg:naive} processes the singleton $I_2=\{3\}$ via the \oneMAPFC\ oracle (Cor.~\ref{cor:single-agent-optimality2}, see Ex.~\ref{ex:single-agent}), returning $\pi^{\star 3}=M^3$ (cost~$0$). The non-trivial component $I_1=\{1,2\}$ is emitted as a self-contained sub-instance and handed to \textsc{JointSolve}; we resolve it in Ex.~\ref{ex:joint} below. The LB/UB seeds are $\sum_i\OPT(\{i\})=2$ and $\cost(M)=8$.
\end{example}

\section{Per-Agent Solver for Singleton Components}
\label{sec:single-agent}

In this section we fill in the per-agent primitive \oneMAPFC\ that the decomposition pipeline of Sec.~\ref{sec:method} dispatches each singleton component to. Given a graph $G$ and a single-agent plan $\pi$ on $G$ (corresponding to $\pi=M^i$ for some agent $i$), \oneMAPFC\ returns a minimum-cost plan obtained from $\pi$ by a sequence of \Collapse\ applications, in time linear in $|\pi|$.

Given a graph $G$ and a single-agent plan $\pi$ on $G$, the \myemph{Collapse DAG}~$D_{G,\pi}=(\mathcal N,\mathcal A)$ is defined as follows. The nodes of~$D_{G,\pi}$ are
$$
\mathcal N\;:=\;\{0,1,\ldots,|\pi|\},
$$
and its arcs are $\mathcal A\;:=\;\mathcal A_{\text{trace}}\cup\mathcal A_{\text{collapse}}$ with
$$
\mathcal A_{\text{trace}}\;:=\;\bigl\{(t,t{+}1)\,\big\vert\,0\le t<|\pi|\bigr\}
$$
and
$$
\mathcal A_{\text{collapse}}\;:=\;\bigl\{(a,b)\,\big\vert\,0\le a<b\le|\pi|\text{ and }\pi(a)=\pi(b)\bigr\}.
$$
An arc $(t,t{+}1)\in\mathcal A_{\text{trace}}$ corresponds to following $\pi$ at timestep $t$ and has cost $\mathds{1}[\pi(t)\ne\pi(t{+}1)]$ (one if $\pi$ moves at timestep $t$, zero if $\pi$ waits). An arc $(a,b)\in\mathcal A_{\text{collapse}}$ corresponds to applying $\Collapse(\pi,a,b)$ and has cost zero.

\begin{lemma}[Collapse-DAG correspondence]\label{obs:interval-dag-correspondence}
Every path in~$D_{G,\pi}$ from node $0$ to node $|\pi|$ corresponds to a single-agent plan $\pi'$ obtained from $\pi$ by applying $\Collapse$ along the path's collapse arcs. Furthermore, the cost of the path equals $\cost(\pi')$.
Conversely, every plan obtained from~$\pi$ by a finite sequence of $\Collapse$ applications corresponds to some path in $D_{G,\pi}$ from node~$0$ to node~$|\pi|$.
\end{lemma}

\begin{proof}
Consider first a $0\!\to\!|\pi|$ path $P$, and apply the $\Collapse$ operations of its collapse arcs in left-to-right order. Every such arc $(a,b)$ satisfies $\pi(a)=\pi(b)$; since an application alters positions only in the interior of its interval, and the arcs of $P$ traverse internally-disjoint intervals, the plan to which $\Collapse(\cdot,a,b)$ is applied agrees with $\pi$ at both $a$ and~$b$, so each application is admissible. The resulting plan $\pi'$ waits at the anchor across every collapse arc of $P$ and follows $\pi$ elsewhere. Consequently, a trace arc $(t,t{+}1)$ of $P$ costs one exactly when $\pi'$ moves at time~$t$, collapse arcs cost zero while $\pi'$ waits across them, and the cost of~$P$ equals $\cost(\pi')$.

For the converse, we show by induction on the number of $\Collapse$ applications that every derived plan $\pi'$ waits across pairwise-disjoint intervals $[a_1,b_1],\ldots,[a_m,b_m]$ satisfying $\pi'(t)=\pi(a_j)=\pi(b_j)$ for $t\in[a_j,b_j]$ and follows~$\pi$ elsewhere; the path that takes the collapse arc $(a_j,b_j)$ across each interval and trace arcs elsewhere then corresponds to~$\pi'$. The claim is immediate for zero applications. For the inductive step, apply $\Collapse(\pi',c,d)$ and denote $w:=\pi'(c)=\pi'(d)$. If $c$ lies in an interval $[a_i,b_i]$ then $w=\pi(a_i)$ and we set $c^*:=a_i$; otherwise $w=\pi(c)$ and we set $c^*:=c$. The index $d^*$ is defined symmetrically. The new plan waits at $w$ across all of $[c^*,d^*]$ (on $[c^*,c)$ and $(d,d^*]$ it already did), we have $\pi(c^*)=\pi(d^*)=w$, and every previous interval is either contained in $[c^*,d^*]$ or disjoint from it; the intervals of the new plan therefore again have the claimed form.
\end{proof}

Fig.~\ref{fig:interval-dag} shows the Collapse DAG of the running-example plan $M^1$ and its minimum-cost path.

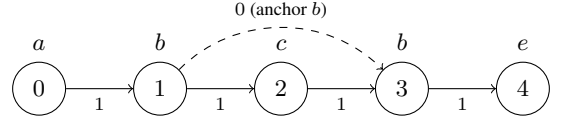
\begin{figure}[t]
\centering
\begin{tikzpicture}[
  vert/.style={circle, draw, minimum size=7mm, inner sep=1pt, font=\small}
]
\node[vert] (n0) at (0, 0)    {$0$};
\node[vert] (n1) at (1.6, 0)  {$1$};
\node[vert] (n2) at (3.2, 0)  {$2$};
\node[vert] (n3) at (4.8, 0)  {$3$};
\node[vert] (n4) at (6.4, 0)  {$4$};
\node[above=1pt, font=\small] at (n0.north) {$a$};
\node[above=1pt, font=\small] at (n1.north) {$b$};
\node[above=1pt, font=\small] at (n2.north) {$c$};
\node[above=1pt, font=\small] at (n3.north) {$b$};
\node[above=1pt, font=\small] at (n4.north) {$e$};
\draw[->] (n0) -- node[below, font=\scriptsize] {$1$} (n1);
\draw[->] (n1) -- node[below, font=\scriptsize] {$1$} (n2);
\draw[->] (n2) -- node[below, font=\scriptsize] {$1$} (n3);
\draw[->] (n3) -- node[below, font=\scriptsize] {$1$} (n4);
\draw[->, dashed, bend left=45] (n1) to node[above, font=\scriptsize] {$0$ (anchor $b$)} (n3);
\end{tikzpicture}
\caption{The Collapse DAG $D_{G,\pi}$ for the plan $\pi=\langle a,b,c,b,e\rangle$ on the running-example graph. Above each node $t$ is the position $\pi(t)$. Solid arrows are trace arcs (cost $\mathds{1}[\pi(t)\ne\pi(t{+}1)]$); the dashed arrow is the unique collapse arc at anchor $b$. The minimum-cost path from $0$ to $4$ is $0\!\to\!1\!\to\!3\!\to\!4$, has cost $2$ and uses the collapse arc.}
\label{fig:interval-dag}
\end{figure}

Since $D_{G,\pi}$ is a DAG, a minimum-cost path from $0$ to~$|\pi|$ is computable in time linear in $|\mathcal A|$ by processing the nodes in increasing index order. In the worst case, however, $|\mathcal A|$ is quadratic in $|\pi|$. \oneMAPFC\ (Alg.~\ref{alg:single-agent}) computes it in $O(|\pi|)$ time and space regardless, by exploiting the structure of the collapse arcs: all collapse arcs entering node~$t$ cost zero and originate at earlier occurrences of the vertex $\pi(t)$, so it suffices to remember, for each vertex, its cheapest occurrence so far and the quadratic arc set is never materialized.

\begin{algorithm}[t]
\caption{\oneMAPFC\ (single-agent \MAPFC): minimum-cost collapse of $\pi$ on $G$.}\label{alg:single-agent}
\begin{algorithmic}[1]
\Require Graph $G$ and single-agent plan $\pi$ on $G$.
\Ensure A min-cost plan $\pi^\star$ obtained from $\pi$ by a sequence of $\Collapse$ applications, with $\cost(\pi^\star)$.
\State $\text{dist}[0]\gets 0$;\ $\text{dist}[t]\gets\infty$ for $t=1,\ldots,|\pi|$
\State $\mathcal B\gets\emptyset$;\ $\mathcal B[\pi(0)]\gets 0$\Comment{argmin hash}
\For{$t=1,2,\ldots,|\pi|$}
  \State $w\gets\mathds{1}[\pi(t{-}1)\ne\pi(t)]$;\ $\text{dist}[t]\gets\text{dist}[t{-}1]+w$
  \State $\text{parent}[t]\gets t{-}1$
  \If{$\pi(t)\in\mathcal B$ \textbf{and} $\text{dist}[\mathcal B[\pi(t)]]<\text{dist}[t]$}
    \State $\text{dist}[t]\gets\text{dist}[\mathcal B[\pi(t)]]$;\ $\text{parent}[t]\gets\mathcal B[\pi(t)]$
  \EndIf
  \If{$\pi(t)\notin\mathcal B$ \textbf{or} $\text{dist}[t]<\text{dist}[\mathcal B[\pi(t)]]$}
    \State $\mathcal B[\pi(t)]\gets t$
  \EndIf
\EndFor
\State $\pi^\star\gets$ path reconstructed via $\text{parent}$ pointers\Comment{Lem.~\ref{obs:interval-dag-correspondence}}
\State \Return $\pi^\star,\ \text{dist}[|\pi|]$
\end{algorithmic}
\end{algorithm}

\oneMAPFC\ receives a graph $G$ and a path $\pi$ and maintains two data structures: (i)~a distance array $\text{dist}:\mathcal N\to\mathbb{R}$, with $\text{dist}[t]$ storing the least cost of a $0$-to-$t$ path found so far; and (ii)~a per-vertex argmin hash $\mathcal B:V\to\mathcal N$, with $\mathcal B[v]=\arg\min_{a:\pi(a)=v}\text{dist}[a]$ storing the visited index of vertex~$v$ of least $\text{dist}$, which lets the best collapse arc into a node be retrieved in $O(1)$. Thus, instead of enumerating the collapse arcs entering node~$t$, the algorithm reads the single entry $\mathcal B[\pi(t)]$, which holds the tail of the cheapest such arc.
We initialize both structures (Lines~1--2), then make a single left-to-right pass over the nodes $t=1,\ldots,|\pi|$ (Line~3).
At node $t$ we first follow the trace arc, setting $\text{dist}[t]\gets\text{dist}[t{-}1]+\mathds{1}[\pi(t{-}1)\ne\pi(t)]$ and $\text{parent}[t]\gets t{-}1$ (Lines~4--5).
If the stored index $\mathcal B[\pi(t)]$ offers a smaller distance, we take the zero-cost collapse arc from it instead, updating $\text{dist}[t]$ and $\text{parent}[t]$ accordingly (Lines~6--7).
We then refresh $\mathcal B[\pi(t)]$ whenever $t$ improves on the stored argmin for vertex $\pi(t)$ (Lines~8--9), so later nodes find the best collapse source in $O(1)$.
Once the pass completes, a backward walk over the parent pointers reconstructs the optimal plan $\pi^\star$ (Line~10), returned with its cost $\text{dist}[|\pi|]$ (Line~11).

\begin{lemma}[single-pass sweep]\label{lem:smart-sweep}
For a graph $G$ and a single-agent plan $\pi$, \oneMAPFC\ (Alg.~\ref{alg:single-agent}) computes a minimum-cost path from $0$ to~$|\pi|$ in the Collapse DAG, using~$O(|\pi|)$ time and space.
\end{lemma}

\begin{proof}
\emph{Correctness.} Nodes are processed in increasing order $t=0,\ldots,|\pi|$, and every arc into $t$ leaves an earlier node $a<t$; thus $\text{dist}[t]$ is computed from already-final values, and this topological-order sweep yields exact distances. Node $t$ has one trace arc $(t{-}1,t)$ and, possibly, several collapse arcs; every collapse arc into $t$ has the form $(a,t)$ with $\pi(a)=\pi(t)$ and cost zero, so the cheapest originates at $\arg\min_{a:\pi(a)=\pi(t)}\text{dist}[a]$, which the hash $\mathcal B$ maintains (Lines~8--9). Setting $\text{dist}[t]$ to the smaller of the trace-arc candidate (Line~4) and this collapse candidate (Lines~6--7) is therefore the minimum over all arcs entering $t$, and the parent pointers encode a corresponding minimum-cost $0$-to-$|\pi|$ path (Line~10).
\emph{Complexity.} Each node costs $O(1)$ amortized work, a constant number of array and hash operations (the arc set $\mathcal A$ is never constructed), and the backward trace is $O(|\pi|)$; the total is $O(|\pi|)$ time and space.
\end{proof}

Combining Lem.~\ref{lem:smart-sweep} with Lem.~\ref{obs:interval-dag-correspondence}, we obtain:

\begin{corollary}[\oneMAPFC\ optimality]\label{cor:single-agent-optimality2}
For a graph $G$ and a single-agent plan $\pi$, a minimum-cost plan obtained from $\pi$ by a sequence of $\Collapse$ applications can be computed in time $O(|\pi|)$.
\end{corollary}

\begin{example}[continuing Ex.~\ref{ex:setup}]\label{ex:single-agent}
We instantiate the construction on each of the three input plans of the running example, with $k_i=4$ for every agent.
\begin{itemize}
    \item For $M^1=\langle a,b,c,b,e\rangle$, $D_{G,M^1}$ (Fig.~\ref{fig:interval-dag}) has four trace arcs of cost~$1$ and a single collapse arc $1\!\to\!3$ (anchor~$b$). The minimum-cost path $0\!\to\!1\!\to\!3\!\to\!4$ has cost~$2$ and yields $\pi^{\star 1}=\langle a,b,b,b,e\rangle$ (Ex.~\ref{ex:single-collapse}).

    \item For $M^2=\langle b,f,b,f,b\rangle$, $D_{G,M^2}$ has all four trace arcs of cost~$1$ and four collapse arcs: $0\!\to\!2$, $0\!\to\!4$, $2\!\to\!4$ (anchor $b$), and $1\!\to\!3$ (anchor $f$). The minimum-cost path is the single-arc path $0\!\to\!4$, applying $\Collapse(M^2,0,4)$, of cost $0$ and corresponding plan $\pi^{\star 2}=\langle b,b,b,b,b\rangle$.

    \item For $M^3=\langle d,d,d,d,d\rangle$, every trace arc has cost $0$ and the minimum-cost path already costs~$0$.
\end{itemize}

The per-agent optima are $\OPT(\{1\})=2$ and $\OPT(\{2\})=\OPT(\{3\})=0$, summing to a lower bound of $2$ on the global optimum of \MAPFC.
\end{example}

\section{A Joint Solver for Non-Trivial Components}
\label{sec:improvements}

In this section we develop a joint solver for the non-trivial components that the decomposition pipeline of Sec.~\ref{sec:method} emits, each a self-contained \MAPFC\ sub-instance on its agents. The framework is agnostic to which solver is plugged in:
\Judgelight\ of \citet{Tang2026Judgelight} applied per sub-instance is one valid configuration (evaluated in Sec.~\ref{sec:experiments}); below we describe an exact \CBS-style joint solver, Collapse-\CBS\ (\CCBS) (Sec.~\ref{subsec:ccbs}) tailored to the non-trivial components, which builds on a constrained single-agent primitive (Sec.~\ref{subsec:constrained-dag}). 
\CCBS\ naturally uses \MAPF-based optimizations, outlined in Sec.~\ref{subsec:optimizations} and detailed in App.~\ref{app:engineering}.

\subsection{The Constrained Single-Agent DAG}
\label{subsec:constrained-dag}

As we will develop a \CBS-style joint solver (Sec.~\ref{subsec:ccbs}), the low-level single-agent planner of Sec.~\ref{sec:single-agent} must account for constraints induced by the high-level search. To this end, we introduce \myemph{negative reach constraints}, each of the form ``$i\notin(v,t)$'', that forbid agent $i$ from occupying cell $(v,t)$ under its plan.
We collect the negative reach constraints on agent~$i$ into a set $\mathcal C_i\subseteq V\times\{0,\dots,T(M)\}$ of \myemph{forbidden cells}, one cell $(v,t)\in\mathcal C_i$ per constraint ``$i\notin(v,t)$'' emitted by the high-level search. Relative to $\mathcal C_i$ we define three notions: (i)~a time $t\in[\first_i(v),\last_i(v)]$ is a \myemph{forbidden time} of anchor~$v$ if $(v,t)\in\mathcal C_i$; (ii)~a \myemph{forbidden-free window} of anchor~$v$ is a maximal interval $W\subseteq[\first_i(v),\last_i(v)]$ containing no forbidden time of $v$; and (iii)~a node $t\in\{0,\ldots,k_i\}$ is \myemph{blocked} if $(M^i(t),t)\in\mathcal C_i$, i.e., $\mathcal C_i$ forbids agent $i$ from its own trajectory cell at time $t$.

We denote $D_i:=D_{G,M^i}=(\mathcal N_i,\mathcal A_i)$ to be the Collapse DAG (Sec.~\ref{sec:single-agent}) of agent~$i$'s input plan, where $\mathcal N_i=\{0,\ldots,k_i\}$ (recall that $k_i=T(M)$ by Asm.~\ref{asm:common-makespan}, so the node indices span the same horizon as the forbidden-cell times). Define the \myemph{constrained Collapse DAG} $\tilde D_i(\mathcal C_i)=(\tilde{\mathcal N}_i,\tilde{\mathcal A}_i)$ as follows: the node set $\tilde{\mathcal N}_i=\{t\in\mathcal N_i \mid t \text{ is not blocked}\}$ excludes the blocked nodes; a trace arc $(t,t{+}1)$ is retained iff both of its endpoints lie in~$\tilde{\mathcal N}_i$; and a collapse arc $a\!\to\! b$ with anchor $v$, corresponding to $\Collapse(M^i,a,b)$, is retained iff $[a,b]$ is contained in some forbidden-free window of $v$.
%

\begin{proposition}[constrained correspondence]\label{prop:constrained-correspondence}
The $0\!\to\! k_i$ paths of $\tilde D_i(\mathcal C_i)$ correspond to the $\Collapse$-derived plans of~$M^i$ satisfying $\mathcal C_i$ and have equal cost. In particular, (i) a minimum-cost $0\!\to\! k_i$ path is a minimum-cost such plan, and (ii) if no $0\!\to\! k_i$ path exists, then no $\Collapse$-derived plan satisfies~$\mathcal C_i$.
\end{proposition}

\begin{proof}
$\tilde D_i(\mathcal C_i)$ is $D_i$ with the blocked nodes and every arc that would violate $\mathcal C_i$ removed, namely the arcs incident to a blocked node and collapse arcs whose window contains a forbidden time. By Lem.~\ref{obs:interval-dag-correspondence} the $0\!\to\! k_i$ paths of $D_i$ correspond to the $\Collapse$-derived plans of $M^i$ at equal cost, so it remains to match the removed arcs with $\mathcal C_i$. Any $0\!\to\! k_i$ path in $\tilde D_i(\mathcal C_i)$ occupies, at each time, either $M^i(t)$ at an unblocked node or an anchor across a forbidden-free window, both avoiding $\mathcal C_i$. Conversely, let~$\pi$ be a $\Collapse$-derived plan satisfying~$\mathcal C_i$ and let~$p$ be its $0\!\to\! k_i$ path in~$D_i$, guaranteed by Lem.~\ref{obs:interval-dag-correspondence}. Every node and arc of~$p$ is retained: (i)~at each node~$t$ visited by~$p$ the plan is at $M^i(t)$, whether~$p$ enters~$t$ along a trace arc or along a collapse arc whose anchor is $M^i(t)$ itself, so $(M^i(t),t)\notin\mathcal C_i$, no visited node is blocked, and every trace arc of~$p$ is retained; and (ii)~a collapse arc $a\!\to\! b$ of~$p$ with anchor~$v$ keeps~$\pi$ at~$v$ throughout~$[a,b]$, so no $t\in[a,b]$ is a forbidden time of~$v$ and $[a,b]$ lies in a forbidden-free window. Hence~$p$ is a path of $\tilde D_i(\mathcal C_i)$. The paths of $\tilde D_i(\mathcal C_i)$ therefore correspond to exactly the $\mathcal C_i$-satisfying plans, at equal cost.
\end{proof}

We extend \oneMAPFC\ (Alg.~\ref{alg:single-agent}, Lem.~\ref{lem:smart-sweep}) to compute a minimum-cost path in $\tilde D_i(\mathcal C_i)$.
Specifically, we start by preprocessing $\mathcal C_i$ to construct, for every time $t$, the set $\mathcal R[t]\subseteq V$ of vertices forbidden at time~$t$. $\mathcal R$ is computed by iterating over all constraints in~$\mathcal C_i$ and for each one updating~$\mathcal R$.
We then process the nodes of~$\tilde D_i(\mathcal C_i)$ in topological order $t=0,\ldots,k_i$, maintaining the distance array $\text{dist}[\,\cdot\,]$ and an argmin hash $\mathcal B:V\to\mathcal N_i$ that, for each vertex, holds the cheapest index within its current forbidden-free window.
If node~$0$ is blocked, i.e., $M^i(0)\in\mathcal R[0]$, the algorithm terminates, as no sequence of $\Collapse$ applications alters the position at time~$0$. Otherwise, we initialize $\text{dist}[0]\gets 0$ and seed $\mathcal B[M^i(0)]\gets 0$.
At each step $t\ge 1$ we perform, in order:
(i)~erase $\mathcal B[v]$ for every $v\in\mathcal R[t]$, so that later visits of~$v$ cannot take a collapse arc from an index before~$t$;
(ii)~if $t$ is blocked, i.e., $M^i(t)\in\mathcal R[t]$, set $\text{dist}[t]\gets\infty$; otherwise set $\text{dist}[t]$ to the smaller of the trace-arc value $\text{dist}[t{-}1]+\mathds{1}[M^i(t{-}1)\ne M^i(t)]$ and the collapse-arc value $\text{dist}[\mathcal B[M^i(t)]]$ (collapse arcs cost zero), recording $\text{parent}[t]$ as the tail of the chosen arc;
and
(iii)~set $\mathcal B[M^i(t)]\gets t$ if $\text{dist}[t]<\text{dist}[\mathcal B[M^i(t)]]$, where an unset entry counts as~$\infty$.
We refer to this procedure as the \myemph{constrained sweep}.

\begin{lemma}[constrained sweep]\label{lem:constrained-sweep}
The constrained sweep computes a minimum-cost $0\!\to\! k_i$ path in $\tilde D_i(\mathcal C_i)$, or reports that none exists, in $O(k_i+|\mathcal C_i|)$ time and space.
\end{lemma}

\begin{proof}
\emph{Correctness.} As in Lem.~\ref{lem:smart-sweep}, the sweep processes nodes in topological order. Node~$0$ is settled at initialization: if it is blocked, it is not a node of $\tilde D_i(\mathcal C_i)$, so no $0\!\to\! k_i$ path exists and the immediate termination is correct; otherwise $\text{dist}[0]=0$ is the cost of the empty path.

Thus, it suffices to show that step~(ii) sets $\text{dist}[t]$ to the minimum of~$\text{dist}[a]$ plus the arc cost, over all arcs $a\!\to\! t$ of $\tilde D_i(\mathcal C_i)$. A blocked~$t$ is not a node of $\tilde D_i(\mathcal C_i)$ and $\text{dist}[t]$ is correctly set to~$\infty$. For an unblocked~$t$, the incoming arcs are the trace arc from $t{-}1$ and the zero-cost collapse arcs whose tails are the indices $a<t$ of the anchor~$M^i(t)$ with no forbidden time of $M^i(t)$ in $[a,t]$; we claim that $\mathcal B[M^i(t)]$ holds the cheapest such tail at the read. Indeed, an index enters $\mathcal B$ at its own iteration only if it improves the incumbent (step~(iii); blocked nodes, whose $\text{dist}$ is $\infty$, never enter), and an entry is erased at every later forbidden time of its vertex (step~(i), those of iteration~$t$ preceding the read); hence exactly the tails not separated from~$t$ by a forbidden time survive, the cheapest one stored. Step~(ii) therefore sets $\text{dist}[t]$ to this minimum; the distances are exact, $\text{dist}[k_i]=\infty$ iff no $0\!\to\! k_i$ path exists, and otherwise the parent pointers trace back a minimum-cost path.

\emph{Complexity.} Constructing $\mathcal R$ takes $O(k_i+|\mathcal C_i|)$ time: initializing its $k_i{+}1$ buckets is $O(k_i)$, and the single scan inserts each cell of $\mathcal C_i$ into its bucket in $O(1)$. Each iteration of the sweep then performs $O(1)$ work beyond the erasures of step~(i), and each cell of $\mathcal C_i$ triggers at most one erasure over the entire sweep, so the erasures total $O(|\mathcal C_i|)$. The structures $\text{dist}$, $\text{parent}$, $\mathcal B$, and $\mathcal R$ occupy $O(k_i+|\mathcal C_i|)$ space.
\end{proof}

\begin{corollary}[constrained \oneMAPFC]\label{cor:constrained-single-agent}
One can compute, in $O(k_i+|\mathcal C_i|)$ time and space, a minimum-cost $\Collapse$-derived plan of $M^i$ satisfying $\mathcal C_i$, or report that no such plan exists.
\end{corollary}

\begin{proof}
The sweep of Lem.~\ref{lem:constrained-sweep} returns a minimum-cost $0\!\to\! k_i$ path in $\tilde D_i(\mathcal C_i)$ or certifies that none exists, within the stated bound; by Prop.~\ref{prop:constrained-correspondence} this is a minimum-cost $\Collapse$-derived plan satisfying $\mathcal C_i$, respectively a certificate that no such plan exists.
\end{proof}

For convenience, we use $\OPT(\tilde D_i(\mathcal C_i))$ to denote the cost of the plan of Cor.~\ref{cor:constrained-single-agent}, setting $\OPT(\tilde D_i(\mathcal C_i))=\infty$ when no such plan exists; with $\mathcal C_i=\emptyset$ it equals $\OPT(\{i\})$.

\subsection{Conflict-Based Joint Solver: \CCBS}
\label{subsec:ccbs}

\CCBS\ adapts \CBS~\citep{Sharon2015CBS} to \MAPFC: given a \MAPFC\ instance $(G,M)$, its search tree (Alg.~\ref{alg:ccbs}) explores constraint sets $\mathcal C=\{\mathcal C_i\}_{i\in I}$, with each node storing the per-agent constrained \oneMAPFC plans (Cor.~\ref{cor:constrained-single-agent}) together with a \myemph{lower bound}
$$
\mathrm{LB}(\mathcal C)\;=\;\sum_{i\in I}\OPT\bigl(\tilde D_i(\mathcal C_i)\bigr).
$$
An \myemph{open node} is one that has been generated but not yet expanded; at each iteration \CCBS\ pops the open node of smallest $\mathrm{LB}$ (Line~4), computes its joint plan (Lines~5--7), and either returns the plan if conflict-free (Lines~10--11) or selects a vertex conflict $(i,j,v,t)$ (Line~12) and \myemph{branches} on it by appending the constraint ``$i\notin(v,t)$'' or ``$j\notin(v,t)$'' in the two children (Lines~13--16); a node whose constrained sub-problem is infeasible is discarded on pop (Lines~8--9).

\begin{algorithm}[t]
\caption{\CCBS: joint solver for a non-trivial component.}\label{alg:ccbs}
\begin{algorithmic}[1]
\Require \MAPFC\ instance $(G,M)$.
\Ensure An optimal feasible joint plan $\Pi^\star$.
\State $\mathcal C^{(0)}\gets\{\mathcal C_i=\emptyset\}_{i\in I}$\Comment{root: no constraints}
\State $\textsc{Open}\gets\bigl\{\bigl(\mathcal C^{(0)},\,\textstyle\sum_i\OPT(\tilde D_i(\emptyset))\bigr)\bigr\}$
\While{$\textsc{Open}\ne\emptyset$}
  \State Pop $(\mathcal C,\mathrm{LB})$ with smallest $\mathrm{LB}$ from \textsc{Open}
  \ForAll{$i\in I$}
    \State build $\tilde D_i(\mathcal C_i)$\Comment{Cor.~\ref{cor:constrained-single-agent}}
    \State $\pi^i\gets$ min-cost $0\!\to\! k_i$ path in $\tilde D_i(\mathcal C_i)$
  \EndFor
  \If{some $\pi^i$ does not exist}
    \State \textbf{continue}\Comment{branch infeasible}
  \EndIf
  \If{$(\pi^i)_{i\in I}$ is collision-free}
    \State \Return $(\pi^i)_{i\in I}$
  \EndIf
  \State Pick a vertex conflict $(i,j,v,t)$ in $(\pi^i)_{i\in I}$
  \State $\mathcal C'\gets\mathcal C$ with $(v,t)$ added to $\mathcal C_i$
  \State Push $\bigl(\mathcal C',\,\sum_\ell\OPT(\tilde D_\ell(\mathcal C'_\ell))\bigr)$ onto \textsc{Open}
  \State $\mathcal C''\gets\mathcal C$ with $(v,t)$ added to $\mathcal C_j$
  \State Push $\bigl(\mathcal C'',\,\sum_\ell\OPT(\tilde D_\ell(\mathcal C''_\ell))\bigr)$ onto \textsc{Open}
\EndWhile
\end{algorithmic}
\end{algorithm}

\begin{theorem}[\CCBS\ soundness and completeness]\label{thm:ccbs}
On every \MAPFC\ instance, \CCBS\ terminates and returns a feasible, optimal joint plan.
\end{theorem}

\begin{proof}
\emph{Soundness.} The standard CBS argument transfers~\citep{Sharon2015CBS}: adding ``$i\notin(v,t)$'' or ``$j\notin(v,t)$'' to a parent node preserves every solution of the parent that avoids the conflict $(v,t)$, since one of the two conflicting agents must vacate that cell in any feasible solution. Hence the optimal feasible plan survives in some leaf. Best-first expansion by $\mathrm{LB}$ guarantees the first conflict-free node popped is optimal, since $\mathrm{LB}$ is admissible: in any feasible joint plan consistent with $\mathcal C$, each agent's plan satisfies $\mathcal C_i$ and thus costs at least $\OPT(\tilde D_i(\mathcal C_i))$. \emph{Completeness.} The constraint space $V\times\{0,\dots,T(M)\}$ is finite, and each branch enlarges one constraint set (the conflict cell is occupied by both conflicting agents' plans and thus, by Prop.~\ref{prop:constrained-correspondence}, lies in neither~$\mathcal C_i$ nor $\mathcal C_j$), so the search tree is finite. The un-collapsed input $M$ is itself a feasible solution, being collision-free with each $M^i$ derived from itself by zero $\Collapse$ applications. Furthermore $M$ survives to a leaf: at any branch on a conflict $(i,j,v,t)$, since $M$ is collision-free agents $i$ and~$j$ are not both at $v$ at time $t$, so $M$ satisfies at least one of the two child constraints ``$i\notin(v,t)$'' and ``$j\notin(v,t)$'' and remains admissible in that child. A feasible leaf therefore exists. Best-first expansion by $\mathrm{LB}$ reaches it.
\end{proof}

\ifshowexamples
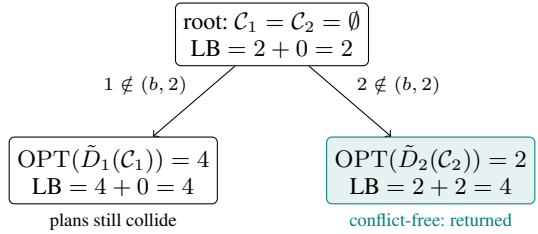
\begin{figure}[t]
\centering
\begin{tikzpicture}[
  ctnode/.style={draw, rounded corners=2pt, align=center, font=\footnotesize, inner sep=3pt},
  lbl/.style={font=\scriptsize, align=center}
]
\node[ctnode] (root) at (0,0) {root: $\mathcal C_1=\mathcal C_2=\emptyset$\\ LB $=2+0=2$};
\node[ctnode] (A) at (-2.1,-1.8) {$\OPT(\tilde D_1(\mathcal C_1))=4$\\ LB $=4+0=4$};
\node[ctnode, draw=teal, fill=teal!10] (B) at (2.1,-1.8) {$\OPT(\tilde D_2(\mathcal C_2))=2$\\ LB $=2+2=4$};
\draw[->] (root) -- node[lbl, above left=-1pt] {$1\notin(b,2)$} (A);
\draw[->] (root) -- node[lbl, above right=-1pt] {$2\notin(b,2)$} (B);
\node[lbl, below] at (A.south) {plans still collide};
\node[lbl, below, text=teal] at (B.south) {conflict-free: returned};
\end{tikzpicture}
\caption{The \CCBS\ constraint tree of Ex.~\ref{ex:joint} on component $I_1=\{1,2\}$. The root's optimal plans collide at $(b,2)$; each child adds one negative constraint. Both children tie at LB~$=4$, but child~A's plans still collide at $(b,1)$ and $(b,3)$, so the search returns child~B's conflict-free plan regardless of tie-breaking.}
\label{fig:ct-tree}
\end{figure}
\fi

\begin{example}[continuing Ex.~\ref{ex:naive}]\label{ex:joint}
\CCBS\ on the running example's non-trivial component $I_1=\{1,2\}$ (Fig.~\ref{fig:ct-tree}) starts at the root with $\mathcal C=\emptyset$ and root LB $=\OPT(\{1\})+\OPT(\{2\})=2+0=2$. The unconstrained joint plan applies $\Collapse(M^1,1,3)$ (giving $\pi^1=\langle a,b,b,b,e\rangle$) and $\Collapse(M^2,0,4)$ (giving $\pi^2=\langle b,b,b,b,b\rangle$), which collide at $(b,1),(b,2),(b,3)$. Branching on the conflict $(1,2,b,2)$ produces two children: child~A forbids agent~$1$ at $(b,2)$ (its $b$-loop is destroyed because $b$ occurs only at $t\in\{1,3\}$, so no forbidden-free window of anchor $b$ contains both endpoints), giving $\OPT(\tilde D_1(\mathcal C_1))=4$ and node LB~$=4$; child~B forbids agent~$2$ at $(b,2)$ (its $b$-loop is destroyed for the same reason), giving $\OPT(\tilde D_2(\mathcal C_2))=2$ via $\Collapse(M^2,1,3)$, with node LB~$=2+2=4$. Expanding child~B yields the joint plan with $\pi^1=\langle a,b,b,b,e\rangle$ (from $\Collapse(M^1,1,3)$) and $\pi^2=\langle b,f,f,f,b\rangle$ (from $\Collapse(M^2,1,3)$); this is collision-free and returned. The total optimum for $I_1$ is $4$; combined with the singleton's $0$, the global optimum is $\OPT(I)=4$.
\end{example}

\subsection{\CBS-Style Optimizations}
\label{subsec:optimizations}
As \CCBS\ (Alg.~\ref{alg:ccbs}) adapts \CBS, many of the algorithmic approaches used to improve the efficiency of \CBS can be applied to \CCBS by slight adaptation to the structure of \MAPFC.
Specifically, we employ
(i)~\myemph{cardinal-conflict prioritization} which branches on conflicts that provably raise the lower bound of both children; 
(ii)~\myemph{disjoint splitting} which replaces the two-child negative split of Alg.~\ref{alg:ccbs} by a negative-vs-positive split on a single agent, cutting the constrained Collapse DAG more aggressively; and 
(iii)~\myemph{conflict bypass} which returns a child immediately when its joint plan is collision-free and matches the parent's lower bound. 
We detail each ingredient, with pseudocode and correctness arguments, in App.~\ref{app:engineering}.

\section{Experimental Evaluation}
\label{sec:experiments}
We evaluate \MAPFCompress\ through the following four questions:
\begin{itemize}
    \item[\textbf{Q1}] What is the computational cost of the per-agent primitive \oneMAPFC\ (i.e., for singleton components)?
    \item[\textbf{Q2}] How does the interaction graph $H$ decompose in practice: what fraction of the instance is singletons, and how large are the non-trivial components?
    \item[\textbf{Q3}] How does \CCBS\ compare with \Judgelight\ on the non-trivial components?
    \item[\textbf{Q4}] How does \MAPFCompress\ compare with \Judgelight\ in solution cost, runtime, and success rate?
\end{itemize}
The underlying premise that motivates our framework is the \myemph{small-component hypothesis}:
on typical instances (which we test through the POGEMA-derived inputs), $H$ is a large mass of singletons plus a residue of small components.
As we will see, the framework's advantage over \Judgelight\ is conditional on this structure.

\subsection{Setup}
\label{subsec:setup}

\paragraph{Problem and platform.} All experiments target \MAPFC\ (Sec.~\ref{sec:formulation}) on the POGEMA suite~\citep{Skrynnik2024POGEMA}, which provides five families: Random (01-random) and Maze (02-mazes), both $32\times 32$; Warehouse (03-warehouse, $33\times 46$); MovingAI city tiles (04-movingai, $64\times 64$); and Puzzle (05-puzzles, $5\times 5$). 
Following~\citet{Tang2026Judgelight}, we evaluate on this suite; the input schedule of each instance is produced by running POGEMA's \algname{BatchAStarAgent} on the scenario.
The benchmark suite consists of $3{,}296$ scenarios: $768$ each in 01-random, 02-mazes, and 03-warehouse, $512$ in 04-movingai, and $480$ in 05-puzzles.

\paragraph{Algorithms compared.} We instantiate the following baselines and configurations. In configuration names we abbreviate \MAPFCompress\ as \DnC\ and \Judgelight\ as \JL.
\begin{description}
  \item[\textnormal{\algname{NoCollapse}.}] Returns the input $M$ unmodified; its cost $\cost(M)$ is the trivial upper bound on $\OPT(I)$.
  \item[\textnormal{\algname{IndepLB}.}] The per-agent optimum $\sum_i \OPT(\{i\})$; a lower bound on $\OPT(I)$.
  \item[\textnormal{\Judgelight.}] The ILP-based baseline of \citet{Tang2026Judgelight}.
  \item[\textnormal{\DnCJL.}] Our algorithmic pipeline (Alg.~\ref{alg:naive}) solving singletons via \oneMAPFC\ (Alg.~\ref{alg:single-agent}) and non-trivial components via \Judgelight.
  \item[\textnormal{\DnCCCBS.}] Our algorithmic pipeline (Alg.~\ref{alg:naive}) solving singletons via \oneMAPFC\ (Alg.~\ref{alg:single-agent}) and non-trivial components via \CCBS\ (using all optimizations described in App.~\ref{app:engineering}; the \CCBS\ configuration used throughout).
  \item[\textnormal{\DnCHyb.}] The configuration we recommend, a regime-aware hybrid. It routes each non-trivial component to the solver that wins in its size regime: (i)~a component with more than $\regimetau$ agents is sent directly to \Judgelight; (ii)~a smaller component is solved by \CCBS\ (as above) under a $1.5$~s budget, falling back to \Judgelight\ on timeout.\footnote{The values $\regimetau$ and $1.5$~s are determined empirically via Q3 (Sec.~\ref{subsec:q3-joint}).}
\end{description}

\paragraph{Metrics.} We report (i)~the \myemph{saving ratio} $1-\mathrm{SoC}/\mathrm{SoC}(M)$, where $\mathrm{SoC}$ denotes the sum-of-costs $\cost(\Pi)$ of Sec.~\ref{sec:formulation}; (ii)~\myemph{total wall-clock} (model construction plus solve); (iii)~the \myemph{singleton fraction}, the fraction of agents in size-one components; and (iv)~the \myemph{success rate}, the fraction of instances a configuration solves within the $30$-second time limit.
\paragraph{Implementation.} All algorithms were implemented in Python, and \Judgelight\ was run from the public release of \citet{Tang2026Judgelight}. \Judgelight\ returns Gurobi's incumbent solution if the per-instance budget expires; in our runs it terminated within the budget on all but two instances (one each in 02-mazes and 03-warehouse), on which it overran the limit but still returned a plan. All experiments were run on an Intel Core Ultra~7 265U ($14$ logical cores, $3.4$~GHz base) running Windows~11 with Python~3.12 and Gurobi~13.0 for the ILP solver. Each (instance, configuration) pair is run five times, where a run covers the complete per-instance pipeline (decomposition, singleton dispatch, and joint solves); the reported wall-clock time is the median over the runs, the reported cost is that of the cheapest successful run (for budget-limited configurations the returned plan can differ across runs near the time limit), and an instance is reported unsolved only when repeated runs fail, in which case the two failing runs settle the outcome and the remaining three are skipped.
Absolute times should nonetheless be read with the implementation in mind: our framework is pure Python, whereas \Judgelight\ delegates its search to a compiled commercial ILP solver, so on per-component solve time our solver is at a disadvantage. 
Code, scenarios, and reproduction scripts are publicly available.\footnote{\url{github.com/CRL-Technion/divide-and-collapse}}

\subsection{Q1: The Per-Agent Primitive Is Near-Instant}
\label{subsec:q1-single-agent}

In this section we evaluate the computational cost of the per-agent primitive \oneMAPFC\ (Q1). Specifically, we run \oneMAPFC\ (Alg.~\ref{alg:single-agent}) in isolation on every agent schedule in the benchmark suite, timed on schedule prefixes spanning $k_i\in\{8,\dots,128\}$ to expose the dependence on the schedule length, for a total of ${\approx}1.5$~million instances. 
The median solve time is $10~\mu$s and no single solve exceeds $150~\mu$s. 
Thus, singleton computation is effectively free from a computational point of view, which, as we will see, is what lets the framework concentrate its time budget on the small coupled residue.

\subsection{Q2: The Interaction Graph Is Mostly Singletons and Small Components}
\label{subsec:q2-decomposition}

\begin{table}[t]
  \centering
  \small
  \setlength{\tabcolsep}{3.5pt}
  \begin{tabular}{lccc}
    \toprule
    Family & Mean $\#$agents & Singleton \% & Median $H$-build \\
    \midrule
    01-random    & $32$  & $45.0\%$ & $1.5$~ms   \\
    02-mazes     & $32$  & $37.9\%$ & $1.8$~ms   \\
    03-warehouse & $112$ & $32.6\%$ & $10.5$~ms  \\
    04-movingai  & $160$ & $37.7\%$ & $15.5$~ms  \\
    05-puzzles   & $3$   & $75.9\%$ & $0.12$~ms  \\
    \midrule
    Overall      & --    & $43.8\%$ & $3.0$~ms   \\
    \bottomrule
  \end{tabular}
  \caption{Per-family decomposition statistics. \emph{Mean \#agents} is the mean number of agents per instance, \emph{Singleton \%} is the per-instance mean of the agent singleton fraction, and \emph{Median $H$-build} the median wall-clock to construct the interaction graph $H$ (Alg.~\ref{alg:naive}).}
  \label{tbl:q2-singleton}
\end{table}

\begin{figure*}[t]
  \centering
  \includegraphics[width=\linewidth]{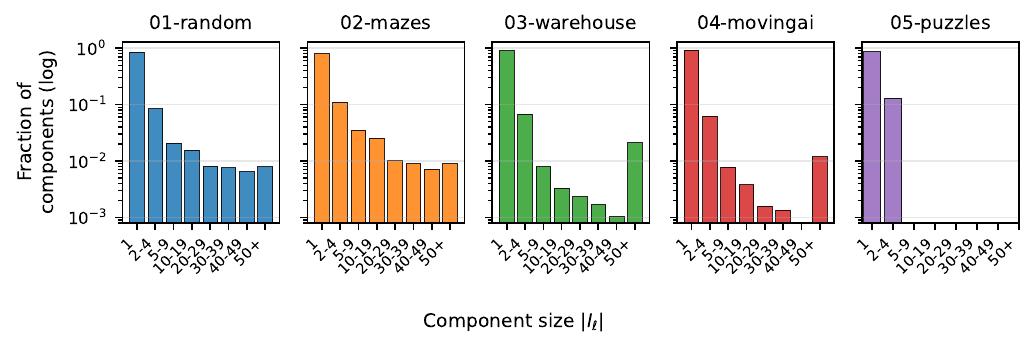}
  \caption{Per-family component-size distribution over the $3{,}296$ instances: the fraction of components in each size bucket, on a log-scale $y$-axis, with singletons in the leftmost (size-$1$) bucket.}
  \label{fig:q2-hist}
\end{figure*}

\begin{table*}[t]
  \centering
  \small
  \begin{tabular}{lrrrrrrr}
    \toprule
    Component size $|I_\ell|$ & $n$ & \CCBS\ success rate & \JL\ success rate & \CCBS\ med. & \JL\ med. & Speedup & Mean saving \\
    \midrule
    $[2, 4]$    & $5{,}236$ & $100\%$  & $100\%$ & $\mathbf{0.26}$~ms   & $4.37$~ms   & $17.0\times$ & $0.271$ \\
    $[5, 9]$    & $890$     & $100\%$  & $100\%$ & $\mathbf{2.52}$~ms   & $13.65$~ms  & $5.4\times$  & $0.389$ \\
    $[10, 19]$  & $542$     & $99.8\%$ & $100\%$ & $\mathbf{20.14}$~ms  & $43.59$~ms  & $2.2\times$  & $0.438$ \\
    $[20, 29]$  & $260$     & $99.6\%$ & $100\%$ & $\mathbf{57.19}$~ms  & $95.60$~ms  & $1.7\times$  & $0.438$ \\
    $[30, 39]$  & $223$     & $96.9\%$ & $100\%$ & $\mathbf{124.49}$~ms & $149.24$~ms & $1.2\times$  & $0.456$ \\
    $[40, 50]$  & $171$     & $93.6\%$ & $100\%$ & $323.84$~ms & $\mathbf{250.22}$~ms & $0.8\times$  & $0.493$ \\
    \midrule
    $[51, 224]$ & $969$ & $38.2\%$ & $100\%$ & $>\!5$~s & $\mathbf{815.4}$~ms & -- & -- \\
    \bottomrule
  \end{tabular}
  \caption{Per-component statistics by size over all $8{,}291$ non-trivial components; $n$ is the number of components per bucket and \emph{success rate} the fraction solved within the $5$-second per-component budget. \emph{Speedup} is the ratio of the two medians shown; \emph{mean saving} is the mean optimal per-component saving ratio $1-\OPT(I_\ell)/\cost(M_{I_\ell})$ over the components solved by both. The faster of the two medians is set in bold. The bottom row aggregates the $969$ components larger than $50$ agents: \CCBS\ exceeds the budget on $61.8\%$ of them, so its median is reported as $>\!5$~s; \emph{Speedup} and \emph{Mean saving}, both defined over the components solved by both, are omitted.}
  \label{tbl:q3-buckets}
\end{table*}

In this section we examine how the interaction graph $H$ decomposes in practice~(Q2), which empirically tests the small-component hypothesis.
Specifically, we run the decomposition pipeline (Alg.~\ref{alg:naive}) on all $3{,}296$ instances and record, per instance, the time to build $H$, the number of agents, the singleton fraction, and the sizes of the non-trivial components; Tbl.~\ref{tbl:q2-singleton} reports per-family statistics and Fig.~\ref{fig:q2-hist} the component-size distribution. 

Building $H$ is effectively free from a computational point of view: the median build time of $3.0$~ms is two orders of magnitude below \Judgelight's median runtime (Sec.~\ref{subsec:q4-framework}), and even on the hardest subset, the $128$ instances with $N{=}256$ (all in 04-movingai), the median rises to ${\approx}0.45$~s, still below \Judgelight's median runtime on that family, with a suite-wide maximum of $0.54$~s.
The mean singleton fraction of $43.8\%$ means that close to half of all agents can be solved through the near-instant \oneMAPFC\ algorithm (Q1). 
However, as half of the agents do not lie in singletons, to support the small-component hypothesis we need to look at the component-size distribution, to which we turn next.

By component count,\footnote{E.g., ten agents forming six singletons plus one four-agent component have an agent singleton fraction of $60\%$ (the statistic reported in Tbl.~\ref{tbl:q2-singleton}) but a component singleton fraction of ${\approx}86\%$ (the statistic reported in Fig.~\ref{fig:q2-hist}).} the decomposition is overwhelmingly trivial: $88\%$ of all components are singletons (Fig.~\ref{fig:q2-hist}). 
The residue itself is small: the $8{,}291$ non-trivial components have a median size of three agents, and an instance hands the joint solver only $0.3$ (05-puzzles) to $5.1$ (04-movingai) components on average: exactly the structure the small-component hypothesis posits. 
The caveat is the tail: the $12\%$ of non-trivial components with $|I_\ell|> 50$, concentrated in 03-warehouse and 04-movingai, hold roughly half of all agents pooled over the suite ($65\%$ on 03-warehouse, $55\%$ on 04-movingai). The framework thus dispatches almost every component trivially, but the agent mass, and hence the runtime, concentrates in a handful of large cores.

\subsection{Q3: \CCBS\ Dominates the Small-to-Medium Residue}
\label{subsec:q3-joint}

In this section we compare \CCBS\ with \Judgelight\ on the non-trivial components (Q3). Specifically, we run both solvers on each of the $8{,}291$ non-trivial components. Both solvers are given a $5$-second per-component time limit. Here, in contrast to the instance-level success rate of Sec.~\ref{subsec:setup}, we report a \myemph{per-component success rate}: the fraction of components solved within the budget. Tbl.~\ref{tbl:q3-buckets} summarizes the results.
\CCBS\ solves~$99.7\%$ of the components of size at most $50$ and \Judgelight\ all of them; where both succeed, \CCBS\ never returns a costlier plan and is lower on $10.8\%$ of them, the result of \Judgelight's greedy safe-oscillation removal (Sec.~\ref{sec:judgelight}); on the components larger than~$50$ agents the effect grows: where both succeed, \CCBS\ produces paths whose cost is cheaper on $325$ of $370$ components, by up to $7.6\%$.

Tbl.~\ref{tbl:q3-buckets} exposes a clean \myemph{phase transition} in component size: \CCBS\ is an order of magnitude faster on the small components that dominate the residue, the advantage narrows as size grows, and \Judgelight\ is faster past the $[30,39]$ band, decisively so beyond $50$ agents, where \CCBS\ solves only~$38.2\%$ within the budget. Noteworthy is that the mean saving \emph{rises} with component size, since larger cores carry more removable oscillation.
We defer an ablation of \CCBS's \MAPF-based optimizations to App.~\ref{app:additional:ablation}.

\subsection{Q4: The Hybrid Matches \Judgelight\ at a Fraction of the Runtime}
\label{subsec:q4-framework}

In this section we compare the full framework with \Judgelight\ (Q4). Specifically, we run all the algorithms described in Sec.~\ref{subsec:setup} on the complete $3{,}296$-instance suite with a time limit of $30$ seconds per instance; 
Fig.~\ref{fig:q4-cactus} shows the runtime distribution and Tbl.~\ref{tbl:q4-aggregate-30s} the aggregate metrics.

\begin{figure}[t]
  \centering
  \includegraphics[width=\linewidth]{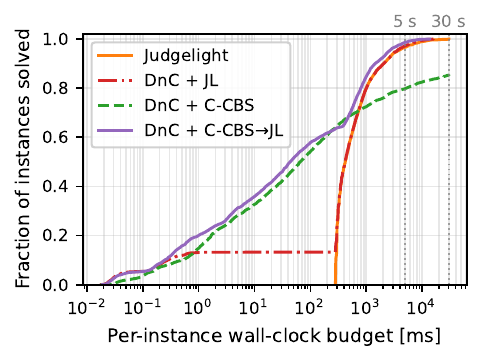}
  \caption{Runtime distribution over the $3{,}296$ instances: cumulative fraction solved against the per-instance wall-clock budget (log scale).}
  \label{fig:q4-cactus}
\end{figure}

Let us first consider \DnCCCBS: 
for small per-instance budgets it solves the largest fraction of instances (Fig.~\ref{fig:q4-cactus}), but its success rate plateaus at $85.4\%$; the unsolved instances contain components on which \CCBS's search tree grows prohibitively large, and increasing the budget does not recover them. 
The hybrid \DnCHyb\ matches \DnCCCBS\ up to that plateau and continues to a success rate of $100\%$ because every component is ultimately handed to \Judgelight: one larger than $\regimetau$ agents is routed there directly, while a smaller one falls back to it whenever \CCBS\ exhausts its $1.5$~s budget.
In contrast, for small per-instance budgets \Judgelight\ and \DnCJL\ solve a smaller fraction of instances than the hybrid. This behavior arises because the ILP solver constructs a model of up to $O(k_i^2)$ variables and $O(k_i^4)$ constraints before search begins, incurring several hundred milliseconds, a cost that \DnCCCBS\ never pays and that accounts for much of the framework's advantage on coordination-light instances. 
Decomposition alone does not accelerate \Judgelight: \DnCJL\ matches \Judgelight\ almost exactly (Tbl.~\ref{tbl:q4-aggregate-30s}), because the per-component model-construction overhead offsets the singleton savings; the speedup comes only from replacing the ILP with \CCBS\ on the non-trivial components.

Under the $30$-second time limit, the hybrid solves all $3{,}296$ instances, and it does so at \Judgelight's solution quality: it matches the mean saving ($0.360$ vs.\ $0.359$). Specifically, it ties with \Judgelight\ on $2{,}697$ instances ($82\%$), improving on $599$ ($18\%$), and losing on none.
Two speedup measures are worth distinguishing: the median per-instance speedup (the median of the per-instance ratios \Judgelight$/$hybrid) is $10.5\times$, and the ratio of the median runtimes is $9.7\times$. The advantage holds per family as well: the recommended configuration is no slower than \Judgelight\ on any of the five families (App.~\ref{app:additional:perfamily}).
Finally, on the $437$ coordination-free instances, where all agents are singletons, the framework reduces to per-agent dispatch (\algname{IndepLB}) and is ${\approx}1{,}900\times$ faster than \Judgelight.

\begin{table}[t]
  \centering
  \small
  \setlength{\tabcolsep}{3pt}
  \begin{tabular}{lrrr}
    \toprule
    Configuration & Mean saving & Median time & Success rate \\
    \midrule
    \multicolumn{4}{@{}l}{\emph{Reference bounds}}\\
    \algname{NoCollapse}                                & $0.000$ & $0.0$~ms    & -- \\
    \algname{IndepLB} (cost LB)                         & $0.393$ & $1.6$~ms    & -- \\
    \midrule
    \multicolumn{4}{@{}l}{\emph{Solvers}}\\
    \Judgelight                                         & $0.359$ & $420.9$~ms  & $99.9\%$  \\
    \DnCJL\                                         & $0.359$ & $421.3$~ms  & $99.9\%$  \\
    \DnCCCBS                                        & $0.369$ & $65.7$~ms   & $85.4\%$  \\
    \DnCHyb\                       & $0.360$ & $43.4$~ms   & $100.0\%$ \\
    \bottomrule
  \end{tabular}
  \caption{Aggregate metrics over the $3{,}296$ instances under the $30$-second time limit. Mean saving is averaged over the instances a configuration solves; median time and success rate are over all $3{,}296$, so the two populations coincide for every row except those that do not solve all instances: \DnCCCBS\ ($2{,}815/3{,}296$), \Judgelight\ ($3{,}294/3{,}296$), and \DnCJL\ ($3{,}295/3{,}296$); \Judgelight\ and \DnCJL\ still return a plan on the instances they exceed the limit on. The first two rows are reference bounds, not solvers: \algname{NoCollapse} is the untouched input (zero saving) and \algname{IndepLB} sums the per-agent optima, an upper bound on achievable saving that is generally jointly infeasible; a success rate is therefore undefined for them.}
  \label{tbl:q4-aggregate-30s}
\end{table}
\section{Discussion and Future Work}\label{sec:future-work}

This work reformulated \MAPFC\ as an exact decomposition over the connected components of the interaction graph, reducing an \NP-hard joint problem to a mass of linear-time singleton solves plus a small residue of joint sub-instances, with \CCBS\ as an exact solver for the residue. Unfortunately, \CCBS's search tree still grows prohibitively on instances with large component sizes.

To this end, the first line of future work is adapting additional \MAPF\ machinery to \MAPFC. This adaptation is non-trivial, since each technique interacts with the constrained Collapse DAG of Prop.~\ref{prop:constrained-correspondence} in a non-standard way. We see the following natural candidates:
(i)~\myemph{interval constraints}, which forbid a vertex over a whole window and would sparsify conflicts on shared cells; here the low level requires no change (a window is a set of forbidden cells) and the challenge is a high-level branching rule that preserves optimality; and
(ii)~\myemph{mutex propagation}~\citep{Zhang2020Mutex}, which prunes jointly-infeasible moves, but whose mutex structure depends on the per-agent collapse availability that is itself a search variable.

The second line concerns problem variants. While we study \MAPFC\ as a one-shot offline post-processing step, its value compounds in \myemph{lifelong} and online settings~\citep{Skrynnik2024Follower}, where plans are continually re-issued and a collapse step sits on the replanning loop. Here, the key question would be how to reuse information from previous search episodes to speed up the planner's running times on hard instances while maintaining the fast running times of microseconds on the easier instances.

\bibliography{aaai2026,mapf_compress}

\appendix

\section{\MAPF-Based Optimizations for \CCBS}\label{app:engineering}

This appendix details the three \MAPF-based optimizations deferred from Sec.~\ref{sec:improvements}: cardinal-conflict prioritization (App.~\ref{app:engineering:cardinal}), disjoint splitting (App.~\ref{app:engineering:disjoint}), and conflict bypass (App.~\ref{app:engineering:bypass}); each adapts a classical \CBS\ optimization~\citep{Sharon2015CBS,Boyarski2015ICBS} to the plans obtainable from the fixed input $M$ by \Collapse\ operations, and the ablation of App.~\ref{app:additional:ablation} measures the empirical contribution of each.

\subsection{Cardinal-Conflict Prioritization}\label{app:engineering:cardinal}

The order in which \CCBS\ branches on conflicts does not affect optimality but may have a substantial effect on the size of the search tree. Following the notion of \myemph{cardinal conflicts}~\citep{Boyarski2015ICBS}, introduced for Improved \CBS\ (\algname{ICBS}), we classify each conflict in a node by the lower-bound increments it induces on the two agents involved, and bias branching towards conflicts that raise the lower bound on both children.

\subsubsection*{Definitions and structural observation}

Let $n=(\mathcal C,\mathrm{LB})$ be a \CCBS\ node, with $\pi^i$ the minimum-cost plan returned by Cor.~\ref{cor:constrained-single-agent} on the constraint set $\mathcal C_i$ and $\Pi=(\pi^i)_{i\in I}$ the joint plan computed at~$n$, and let $(i,j,v,t)$ be a vertex conflict in~$\Pi$ (edge conflicts cannot arise, Lem.~\ref{lem:no-edge-conflicts}). Define the two child LB increments
\begin{align*}
\Delta_i &\;=\; \OPT\!\bigl(\tilde D_i(\mathcal C_i \cup \{(v,t)\})\bigr) - \OPT\!\bigl(\tilde D_i(\mathcal C_i)\bigr), \\
\Delta_j &\;=\; \OPT\!\bigl(\tilde D_j(\mathcal C_j \cup \{(v,t)\})\bigr) - \OPT\!\bigl(\tilde D_j(\mathcal C_j)\bigr).
\end{align*}
Branching on $(i,j,v,t)$ produces two children, one forbidding $(v,t)$ for agent~$i$ and one for agent~$j$; their LBs equal $\mathrm{LB}+\Delta_i$ and $\mathrm{LB}+\Delta_j$, respectively. We say that the conflict is (i)~\myemph{cardinal} if both $\Delta_i$ and $\Delta_j$ are positive, (ii)~\myemph{semi-cardinal} if exactly one of $\Delta_i,\Delta_j$ is positive, and (iii)~\myemph{non-cardinal} if $\Delta_i=\Delta_j=0$.

Conceptually, branching on a cardinal conflict increases the LB of both children, whereas branching on a non-cardinal conflict leaves at least one child with the parent's LB intact and may chain into long non-progressing subtrees. 
Importantly, cardinal conflicts are common in \MAPFC: a forbidden cell $(v,t)$ in the interior of a forbidden-free window $W$ of $v$ (Prop.~\ref{prop:constrained-correspondence}) splits $W$ into two, destroying every collapse arc of $v$ that spans the split point; in particular, when $W=[\first_i(v),\last_i(v)]$ the widest collapse arc of $v$ disappears, and a single constraint can raise the per-agent cost by up to the number of moves the destroyed arcs covered.

\subsubsection*{Algorithm and analysis}

To this end, Alg.~\ref{alg:eng-cardinal} implements cardinal-first conflict selection at a single \CCBS\ node. The procedure scans the vertex conflicts of the joint plan in earliest-time order, classifying each by evaluating the constrained sweep on both candidate children at a cost of $O(k_i+|\mathcal C_i|+k_j+|\mathcal C_j|)$ per conflict examined (Cor.~\ref{cor:constrained-single-agent}). It returns the first cardinal conflict it encounters, falling back to the first semi-cardinal conflict and, absent those, to the first non-cardinal one. Note that the terms $\OPT(\tilde D_i(\mathcal C_i))$, which do not depend on the conflict, are computed once per node and cached.

\begin{algorithm}[t]
\caption{\textsf{SelectConflict}: cardinal-first conflict selection at a \CCBS\ node.}\label{alg:eng-cardinal}
\begin{algorithmic}[1]
\Require Node $n=(\mathcal C,\mathrm{LB})$ with the per-agent plans $\pi^i$ of Cor.~\ref{cor:constrained-single-agent} and joint plan $\Pi=(\pi^i)_{i\in I}$
\Ensure Selected vertex conflict $(i^\star,j^\star,v^\star,t^\star)$ or $\bot$ if $\Pi$ is conflict-free
\State $\mathsf{semi} \gets \bot$; $\mathsf{non} \gets \bot$
\For{each vertex conflict $(i,j,v,t)$ of $\Pi$ in earliest-time-then-lexicographic order} \label{ln:eng-card-enum}
    \State $\Delta_i \gets \OPT(\tilde D_i(\mathcal C_i \cup \{(v,t)\})) - \OPT(\tilde D_i(\mathcal C_i))$
    \State $\Delta_j \gets \OPT(\tilde D_j(\mathcal C_j \cup \{(v,t)\})) - \OPT(\tilde D_j(\mathcal C_j))$
    \If{$\Delta_i > 0$ \textbf{and} $\Delta_j > 0$}
        \State \Return $(i,j,v,t)$ \Comment{first cardinal conflict}
    \ElsIf{$\Delta_i > 0$ \textbf{or} $\Delta_j > 0$}
        \If{$\mathsf{semi} = \bot$} $\mathsf{semi} \gets (i,j,v,t)$ \EndIf
    \Else
        \If{$\mathsf{non} = \bot$} $\mathsf{non} \gets (i,j,v,t)$ \EndIf
    \EndIf
\EndFor
\State \Return $\mathsf{semi}$ if $\mathsf{semi}\ne\bot$, otherwise $\mathsf{non}$
\end{algorithmic}
\end{algorithm}

\begin{lemma}[Cardinal-first preserves \CCBS\ guarantees]\label{lemma:eng-cardinal-optimality}
Replacing the default conflict-selection rule of \CCBS\ (Alg.~\ref{alg:ccbs}) by \textsf{SelectConflict} of Alg.~\ref{alg:eng-cardinal} preserves the soundness and completeness guarantees of Thm.~\ref{thm:ccbs}.
\end{lemma}

\begin{proof}
Conflict selection in \CCBS\ enters the algorithm only after a node has been popped from the LB-ordered open list, and only to decide which conflict to branch on. The two children produced by branching on $(i,j,v,t)$ are $\mathcal C \cup \{i \notin (v,t)\}$ and $\mathcal C \cup \{j \notin (v,t)\}$ regardless of how the conflict was selected, and Thm.~\ref{thm:ccbs} shows that this binary split is sound and complete for any choice of conflict in $\Pi$. Cardinal-first only changes which conflict is branched on, not the branching mechanism itself. The best-first expansion on LB therefore continues to return an optimum on termination, and the finite state space guarantees termination.
\end{proof}

\begin{example}[continuing Ex.~\ref{ex:joint}]\label{ex:eng-cardinal}
Consider the vertex conflict $(1,2,b,2)$ at the root node: the unconstrained optima keep agent~$1$ waiting at $b$ throughout $[1,3]$ and agent~$2$ throughout~$[0,4]$, so both occupy $(b,2)$. For agent~$1$, forbidding $(b,2)$ splits the single window $[1,3]$ of its only anchor $b$ into~$[1,1]$ and $[3,3]$, destroying the collapse arc, so it must keep the moves $b\!\to\!c\!\to\!b$: its minimum cost rises from $2$ to $4$, giving $\Delta_1=2$. For agent~$2$, the same constraint splits the $b$-window $[0,4]$ into $[0,1]$ and $[3,4]$ and destroys the arc~$0\!\to\!4$, but its second anchor $f$ (window $[1,3]$) still admits a collapse of cost $2$, so its minimum cost rises from $0$ to $2$ and $\Delta_2=2$. Both increments are strictly positive: the conflict is cardinal, and \textsf{SelectConflict} branches on it, raising the LB from $2$ to $4$ on both children, in agreement with the trace of Ex.~\ref{ex:joint}.
\end{example}

\subsection{Disjoint Splitting}\label{app:engineering:disjoint}

The default \CCBS\ branching on a conflict $(i,j,v,t)$ adds the negative constraint $i \notin (v,t)$ in one child and $j \notin (v,t)$ in the other. Following the disjoint-splitting variant of \CBS~\citep{Li2019DisjointSplitting}, we replace this rule by splitting on a single chosen agent, say $i$: the first child carries the negative constraint $i \notin (v,t)$, and the second child carries the positive constraint $i \in (v,t)$, which forces $i$ to occupy $(v,t)$. In the positive child, every other agent whose current plan occupies $(v,t)$ additionally receives the matching negative constraint: agent~$i$ occupies $(v,t)$ in every plan consistent with that child, so no conflict-free solution in its subtree places another agent there, and these constraints prune only colliding plans.

The constrained Collapse DAG is defined for \emph{negative} constraints only, so we instantiate a  positive constraint  by reducing it to a set of negative constraints the primitive already handles. For agent $i$ and time $t$, let
$$
R_i(t)\;=\;\{M^i(t)\}\,\cup\,\{\,w:\ \first_i(w)\le t\le\last_i(w)\,\},
$$
where $w$ ranges over the anchors of~$M^i$, be the finite set of vertices $i$ can occupy at time $t$ under some $\Collapse$-derived plan: either it stays on its trajectory at $M^i(t)$, or it waits at an anchor $w$ with $\first_i(w)\le t\le\last_i(w)$. $R_i(t)$ is the time-$t$ slice of the reach set $\Reach^i$ of Sec.~\ref{sec:method}, so every $\Collapse$-derived plan satisfies $\pi^i(t)\in R_i(t)$ by Obs.~\ref{obs:reach-containment}.

\begin{lemma}[positive constraints reduce to negative ones]\label{lem:positive-constraint}
For a target vertex $v$, the positive constraint $\pi^i(t)=v$ holds for a $\Collapse$-derived plan if and only if that plan satisfies the negative constraint set $\mathcal C^+_i(v,t)=\{(u,t):u\in R_i(t)\setminus\{v\}\}$.
\end{lemma}

\begin{proof}
Every $\Collapse$-derived plan has $\pi^i(t)\in R_i(t)$. If $v\in R_i(t)$, forbidding every $u\in R_i(t)\setminus\{v\}$ at time~$t$ leaves $\pi^i(t)=v$ as the only admissible value. If $v\notin R_i(t)$, the same constraint set forbids all of $R_i(t)$ and the sub-problem becomes infeasible, and indeed no $\Collapse$-derived plan places $i$ at~$v$ at time~$t$.
\end{proof}

The forbidden set $\mathcal C^+_i(v,t)$ contains fewer than $k_i$ cells, all at the single time $t$, so the positive child is an ordinary constrained Collapse DAG: it is solved, and its feasibility decided, in time $O(k_i+|\mathcal C_i|)$ by Cor.~\ref{cor:constrained-single-agent}, and $\OPT\bigl(\tilde D_i(\mathcal C_i\cup\mathcal C^+_i(v,t))\bigr)$ is an admissible per-agent lower bound. Consequently, the disjoint split inherits the guarantees of Thm.~\ref{thm:ccbs}: the two children cover every $\Collapse$-derived plan (agent~$i$ either occupies $(v,t)$ or it does not), each child's $\OPT$ is a valid lower bound on its restricted plan set, and best-first expansion on $\mathrm{LB}$ still returns an optimum.

Termination is preserved as well. The negative child strictly enlarges $\mathcal C_i$ (the plan of agent~$i$ occupies $(v,t)$, so $(v,t)\notin\mathcal C_i$); the positive child enlarges either $\mathcal C_i$ (by~$\mathcal C^+_i(v,t)$) or, when $\mathcal C^+_i(v,t)\subseteq\mathcal C_i$ already holds, the constraint set of the other conflict agent through the propagated constraint (its plan also occupies $(v,t)$). Every branch thus enlarges some agent's constraint set within the finite cell space, and the finiteness argument of Thm.~\ref{thm:ccbs} applies verbatim.

\subsection{Conflict Bypass}\label{app:engineering:bypass}

\algname{ICBS}~\citep{Boyarski2015ICBS} bypasses a conflict by letting the parent adopt a child plan of equal cost with fewer conflicts, avoiding the split altogether. We use a restricted form of this idea, tailored to \CCBS. When \CCBS\ expands a node $n$ with conflict $(i,j,v,t)$ into children $n_i$ and~$n_j$, the bypass rule is simple: if the joint plan reconstructed at a child is collision-free and its cost matches the parent LB, we return it immediately rather than continuing to branch. Since~$n$ is popped in best-first order, its lower bound is the minimum key of the open list at that moment and hence a lower bound on the optimum of the instance being solved, so a collision-free plan attaining it is optimal. A child whose plan is collision-free but more expensive than the parent does not trigger bypass and is processed normally by \CCBS's branch-and-bound loop. In our experiments this rule fires frequently; its effect is to trim per-node reconstruction work rather than tree size (App.~\ref{app:additional:ablation}).

\section{Additional Experiments}\label{app:additional}

This appendix describes two supporting experiments deferred from Sec.~\ref{sec:experiments}: the ablation of \CCBS's \MAPF-based optimizations (App.~\ref{app:additional:ablation}) and the per-family end-to-end breakdown (App.~\ref{app:additional:perfamily}). All use the platform, baselines, and protocol of Sec.~\ref{subsec:setup}.

\subsection{Ablation of \CCBS's \MAPF-Based Optimizations}\label{app:additional:ablation}

\CCBS's efficiency on the non-trivial residue rests on the three \MAPF-based optimizations of App.~\ref{app:engineering}. To quantify each, we isolate them with an incremental ladder over the $7{,}322$ components of size at most $50$ from Sec.~\ref{subsec:q3-joint} (Fig.~\ref{fig:q3-ablation}): vanilla \CCBS\ (earliest-time conflict selection, classical branching), then cardinal-conflict prioritization (App.~\ref{app:engineering:cardinal}), disjoint splitting (App.~\ref{app:engineering:disjoint}), and conflict bypass (App.~\ref{app:engineering:bypass}), the last being our default; all rungs, vanilla included, run with two exactness-preserving implementation accelerations (deferring a child's constrained solves to pop time under its parent's bound, and maintaining the conflict map incrementally across CT nodes), so the ladder varies only the three search ingredients. We score each variant by its median high-level expansion count, a machine-independent proxy for search effort that, unlike wall-clock, is unaffected by the per-component budget. As in Sec.~\ref{subsec:q3-joint}, success rates are per component, namely the fraction of the $7{,}322$ components solved within the $5$-second budget.

Vanilla \CCBS\ already solves $94.9\%$ of the components, but its search tree grows by orders of magnitude with component size and its timeouts concentrate on the largest cores: the failure mode is tree explosion, not slow node processing. Cardinal-conflict prioritization is the decisive ingredient, shrinking the median expansion count on the large buckets by an order of magnitude and lifting the success rate to $99.3\%$; this is the behavior the structural observation of App.~\ref{app:engineering:cardinal} predicts, since a single forbidden cell often destroys an entire collapse arc, making conflicts that provably raise the lower bound both common and highly informative. Disjoint splitting contributes the remaining coverage on the medium-to-large buckets, reaching the $99.7\%$ success rate of the production solver, whereas conflict bypass leaves the tree essentially unchanged at this scale: its value lies in the per-node work it avoids, not in the nodes it prunes. All three \MAPF-based optimizations preserve optimality by construction.

\begin{figure}[t]
  \centering
  \includegraphics[width=\linewidth]{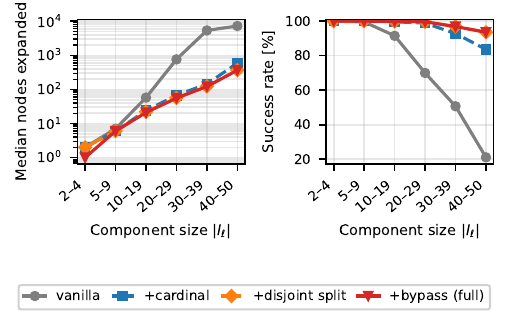}
  \caption{Ablation of \CCBS's \MAPF-based optimizations.
  Left: median high-level nodes expanded by component size (log scale). Right: success rate within the $5$~s per-component budget.}
  \label{fig:q3-ablation}
\end{figure}

\subsection{Per-Family End-to-End Breakdown}\label{app:additional:perfamily}

Tbl.~\ref{tbl:q4-per-family-30s} breaks the Q4 comparison of the recommended \DnCHyb\ against \Judgelight\ (Sec.~\ref{subsec:q4-framework}) down by POGEMA family and isolates the contribution of the regime-aware routing. \DnCHyb's median-runtime advantage is largest on the singleton-dominated families ($\approx\!1{,}200\times$ on 05-puzzles, $21\times$ on 01-random, and $11\times$ on 02-mazes) and it also leads on the two families whose residue holds the largest cores, by $1.5\times$ on 03-warehouse and $1.1\times$ on 04-movingai. Regime-aware routing is what secures the latter: the regime-agnostic variant, which runs every non-trivial component through the $1.5$-second \CCBS\ budget before falling back, is instead ${\approx}3.2\times$ slower than \Judgelight\ on 04-movingai, since \CCBS\ exhausts that budget on the largest cores before \Judgelight\ finishes the job. Regime-aware routing also stabilizes the returned cost, leaving no instance whose reported cost flips across timing repeats, versus twenty under the regime-agnostic variant. Saving is essentially unchanged across all families and both variants, the two agreeing to within $5\cdot 10^{-4}$.

\begin{table}[t]
  \centering
  \small
  \setlength{\tabcolsep}{2.5pt}
  \begin{tabular}{lrrrrr}
    \toprule
    & \multicolumn{2}{c}{\Judgelight} & \multicolumn{2}{c}{\DnCHyb} & \shortstack{No\\routing} \\
    \cmidrule(lr){2-3}\cmidrule(lr){4-5}\cmidrule(lr){6-6}
    Family & Saving & Median & Saving & Median & Median \\
    \midrule
    01-random    & $0.376$ & $344.0$  & $0.378$ & $16.5$   & $18.3$   \\
    02-mazes     & $0.463$ & $377.9$  & $0.464$ & $33.0$   & $39.1$   \\
    03-warehouse & $0.273$ & $1{,}490.8$ & $0.273$ & $1{,}011.8$ & $1{,}432.5$ \\
    04-movingai  & $0.170$ & $640.8$  & $0.170$ & $594.4$  & $2{,}080.9$ \\
    05-puzzles   & $0.507$ & $283.2$  & $0.507$ & $0.2$    & $0.2$    \\
    \bottomrule
  \end{tabular}
  \caption{Per-family comparison of \Judgelight\ and the recommended \DnCHyb: mean saving ratio and median wall-clock (ms) over the full $3{,}296$-instance suite. The final column gives the median wall-clock of the regime-agnostic hybrid (every component sent through \CCBS\ first); the regime-aware hybrid leads \Judgelight\ on every family, whereas the regime-agnostic variant is $3.2\times$ slower on 04-movingai.}
  \label{tbl:q4-per-family-30s}
\end{table}

\end{document}